\documentclass[10pt,twocolumn,letterpaper]{article}

\usepackage[pagenumbers]{cvpr} % To force page numbers, e.g. for an arXiv version

\usepackage{pifont}

\def\v0{{\bf 0}}

\def\bsp{{\boldsymbol p}}
\def\bsq{{\boldsymbol q}}

\def\bst{{\boldsymbol t}}

\def\bsR{{\boldsymbol R}}

\def\cC{{\mathcal C}}

\def\cP{{\mathcal P}}
\def\cQ{{\mathcal Q}}

\DeclareMathOperator*{\argmin}{argmin}

\definecolor{sourceclr}{HTML}{91B1B1}
\definecolor{targetclr}{HTML}{C29211}
\definecolor{rreclr}{HTML}{1F77B4}
\definecolor{rteclr}{HTML}{D62728}

\newtheorem{proposition}{Proposition}[section]
\usepackage{newproof}

\usepackage{graphicx}
\usepackage{epstopdf}
\graphicspath{{./}{Figures/}}
\DeclareGraphicsExtensions{.pdf,.png,.jpg}

\usepackage{booktabs}
\usepackage{multirow}
\usepackage{makecell}
\usepackage{arydshln}

\definecolor{cvprblue}{rgb}{0.21,0.49,0.74}
\usepackage[pagebackref,breaklinks,colorlinks,allcolors=cvprblue]{hyperref}

\title{\textit{C-GenReg}: Training-Free 3D Point Cloud Registration by Multi-View-Consistent Geometry-to-Image Generation with Probabilistic Modalities Fusion}
\author{
Yuval Haitman\thanks{Corresponding author: \texttt{haitman@post.bgu.ac.il}}
\quad
Amit Efraim
\quad
Joseph M. Francos\\
Ben-Gurion University, Beer-Sheva, Israel
}
\usepackage[subtle]{savetrees}
\usepackage{enumitem}

\fussy                 % opposite of \sloppy
\begin{document}
\maketitle
\begin{abstract}
We introduce C-GenReg, a training-free framework for 3D point cloud registration that leverages the complementary strengths of world-scale generative priors and registration-oriented Vision Foundation Models (VFMs). Current learning-based 3D  point cloud registration methods struggle to generalize across sensing modalities, sampling differences, and environments. Hence, C-GenReg augments the geometric point cloud registration branch by transferring the matching problem into an auxiliary image domain, where VFMs excel, using a World Foundation Model to synthesize multi-view-consistent RGB representations from the input geometry. This generative transfer preserves spatial coherence across source and target views without any fine-tuning. From these generated views, a VFM pretrained for finding dense correspondences extracts matches. The resulting pixel correspondences are lifted back to 3D via the original depth maps. 
To further enhance robustness, we introduce a “Match-then-Fuse” probabilistic cold-fusion scheme that combines two independent correspondence posteriors, that of the generated-RGB branch with that of the raw geometric branch. This principled fusion preserves each modality's inductive bias and provides calibrated confidence without any additional learning. C-GenReg is zero-shot and plug-and-play: all modules are pretrained and operate without fine-tuning. Extensive experiments on indoor (3DMatch, ScanNet) and outdoor (Waymo) benchmarks demonstrate strong zero-shot performance and superior cross-domain generalization. For the first time, we demonstrate a generative registration framework that operates successfully on real outdoor LiDAR data, where imagery is unavailable. Code is available at: \noindent\url{https://github.com/yuvalH9/CGenReg}
\end{abstract}    
\vspace{-17pt}
\section{Introduction}
\label{sec:intro}
\begin{figure}
    \centering
    \includegraphics[width=1\linewidth, trim={0mm 8mm 0mm 0mm}, clip]{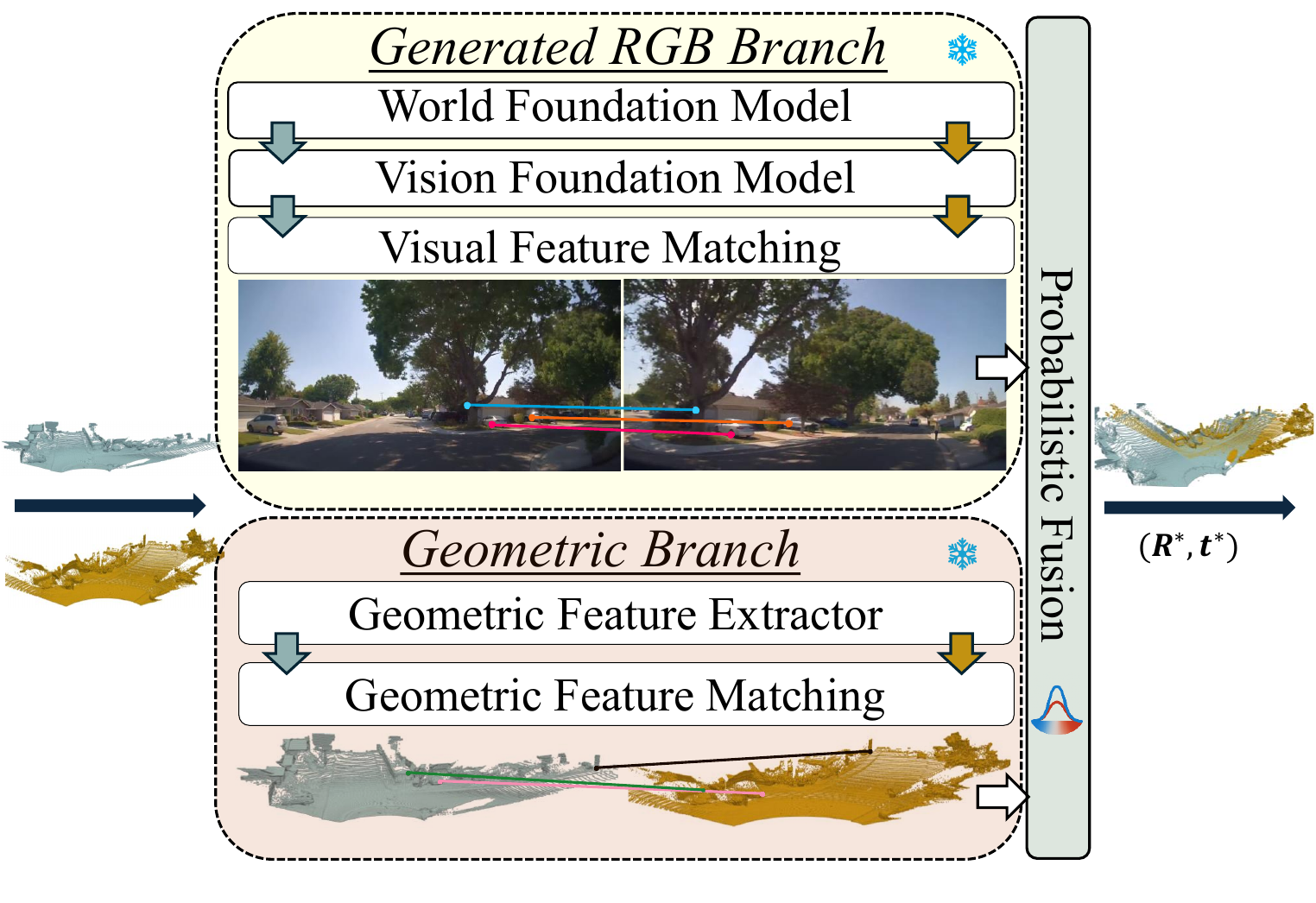}
    \caption{\textbf{\textit{C-GenReg}}: A training-free point cloud registration framework. The pipeline operates in two parallel branches: (1) \textbf{Generated-RGB Branch} - a World Foundation Model generates RGB views that are geometrically aligned with the input \textbf{\textcolor{sourceclr}{source}} and \textbf{\textcolor{targetclr}{target}} point clouds and visually consistent across the two viewpoints; a task-specific Vision Foundation Model extracts dense image features and estimates RGB-based correspondences. (2) \textbf{Geometric Branch} - a geometric feature extractor encodes structural cues directly from the raw 3D point clouds and independently produces geometry-based correspondences. The two correspondence probability maps are then fused using our ``\textit{Match-then-Fuse}" probabilistic fusion to yield the final correspondence set for estimating the rigid transformation aligning the two point clouds.}
    \vspace{-15pt}
    \label{fig:teaser}
    
\end{figure}

Point cloud registration is the procedure of aligning multiple 3D scans into a single coordinate system by estimating a rigid transformation between source and target point clouds — a key step in many vision and robotics tasks.
Standard point cloud registration consists of feature extraction, feature matching, and robust pose estimation (\eg RANSAC \cite{ransac1981}). Although learned 3D feature extractors \cite{3dmatch2017,fcgf2019,predetor2021,geotransformer2022} have replaced handcrafted descriptors \cite{fpfh2009,shot2010}, the pipeline itself remains unchanged, and performance is still limited primarily by imprecise feature matching.

Despite recent progress, learned 3D feature extractors remain strongly domain-dependent: their performance varies significantly with sensing modality, point density, and acquisition settings. Methods that perform well in indoor RGB-D scenes often degrade on different sensors or outdoor LiDAR data, revealing limited cross-domain generalization.
In contrast, the image domain has largely overcome such generalization limits through Vision Foundation Models (VFMs), which achieve remarkable robustness by training on massive, heterogeneous datasets \cite{vfm_survay}. An analogous foundation model for 3D point clouds is still absent \cite{pc_vfms}. A promising alternative is to transfer 3D geometry into the image modality, \eg generating RGB views from depth or point clouds, so that registration can exploit the strong priors and semantically rich features of pretrained VFMs.

A geometry-to-image transfer is effective for point cloud registration only if the generated RGB views are both (i) multi-view consistent across source and target viewpoints and (ii) geometrically coherent with the underlying 3D structure. Without these properties, the generated images may diverge or introduce geometric distortions, leading to unreliable correspondences. Recent World Foundation Models (WFMs) \cite{cosmos_wfm, cosmos_transfer} encode world-scale priors and multi-view geometric reasoning, enabling them to produce RGB views that remain structurally consistent across viewpoints, making them ideal generative backbones for registration. Importantly, the generated images need not match the true scene appearance, \ie colors and textures may differ, as long as their geometry is preserved across different views.

Existing generative approaches for point cloud registration \cite{gpcr, zeromatch, wang2024freereg} have recently demonstrated the potential of diffusion models for this task. However, these methods primarily rely on single-view generation and lack mechanisms for handling multiple geometrically related views. As a result, they typically require fine-tuning to enforce multi-view consistency; while they can operate without such adaptation, performance often degrades due to inconsistent cross-view geometry. In contrast, our method, \textit{C-GenReg} (stands for \textbf{C}onsistent \textbf{Gen}erative \textbf{Reg}istration), leverages WFMs to generate multi-view–consistent RGB views directly from geometry, eliminating the need for any retraining. Crucially, \textit{C-GenReg} applies a task-specific VFM, trained for dense geometric correspondence estimation \cite{dust3r_cvpr24, mast3r_eccv24, roma}, to these WFM-generated images. This combination produces substantially stronger and more geometrically meaningful correspondences than those obtained with general-purpose VFMs \cite{dino2}, whose representations are not aligned with the matching objective. The task-specific inductive bias reinforces correspondence quality precisely where registration needs it most.

Although transferring 3D geometry to the image domain provides rich visual cues, it does not fully exploit the complementary geometric information present in the original point cloud. To leverage both modalities, we incorporate a principled fusion strategy. Prior works typically use simple feature concatenation \cite{zeromatch, gpcr}, which we find suboptimal as it ignores each modality’s inductive bias and the probabilistic nature of their correspondence predictions. Instead, we introduce a ``\textit{Match-then-Fuse}" scheme that combines two independent correspondence posteriors, one from the WFM + VFM branch and one from the geometric branch, under a conditional-independence assumption. This posterior-level fusion preserves the pretrained structure of each modality, provides calibrated and robust correspondences, and remains fully training-free.

Our approach is fully plug-and-play and operates without any fine-tuning, seamlessly pairing with a wide range of registration-oriented geometric feature extractors and generalizing across both indoor RGB-D and outdoor LiDAR settings. Notably, it is, to the best of our knowledge, the first generative registration framework to operate successfully on real LiDAR data.

The main contributions of this paper are:
\begin{enumerate}
    \item We introduce \textit{C-GenReg}, a two-branch framework that (i) performs geometry-to-RGB transfer using a zero-shot, multi-view-consistent World Foundation Model (WFM) combined with a registration-oriented Vision Foundation Model (VFM), and (ii) extracts geometric features directly from the raw point clouds. These complementary branches provide substantially stronger and more reliable correspondences than relying solely on general-purpose VFMs or geometry alone.
    \item We introduce a probabilistic ``\textit{Match-then-Fuse}" scheme that combines independent correspondence posteriors from the WFM+VFM imagery and the geometric branches, yielding calibrated and robust matches without any learning.
    \item Our pipeline operates fully zero-shot and requires no additional training, including RGB generation, correspondence estimation, and fusion. It relies solely on off-the-shelf pretrained models, making it a plug-and-play module that enhances existing geometric registration pipelines.
    \item We achieve SOTA zero-shot results across indoor RGB-D benchmarks (3DMatch, ScanNet) and, for the first time, demonstrate a generative registration framework that successfully operates on real outdoor LiDAR data (Waymo).
\end{enumerate}
\begin{figure*}
    \centering
    \includegraphics[width=0.9\linewidth]{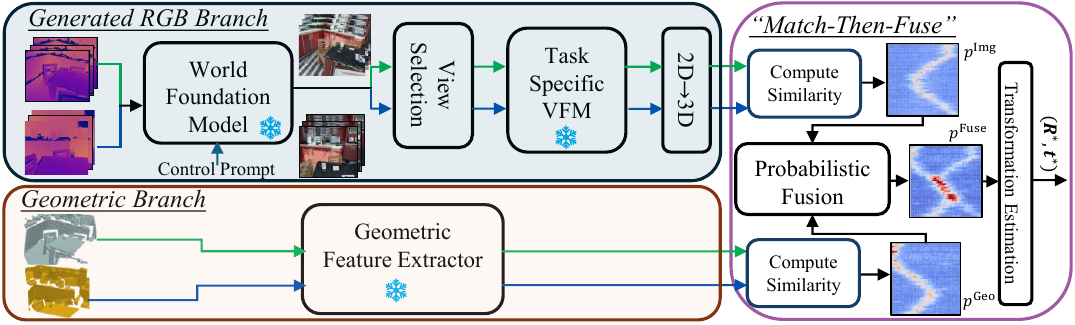}
     \caption{\textbf{\textit{C-GenReg} Overview}: A training-free, zero-shot point cloud registration framework with two parallel branches. (1) \textbf{Generated-RGB Branch} - \textbf{\textcolor{green!50!black}{source}} and \textbf{\textcolor{blue!70!black}{target}} point clouds are each represented as depth-frame sequences, temporally concatenated and processed by a frozen \textit{World Foundation Model} to generate RGB views that are geometrically aligned and appearance-consistent across views. A subset of $K$ frames per domain is fed to a frozen, task-specific \textit{Vision Foundation Model} (VFM) to extract dense pixel-level features, later lifted to 3D using the original depths. (2) \textbf{Geometric Branch} - extracts dense geometric features directly from the raw point clouds using a pretrained geometric feature extractor. Each modality yields a posterior correspondence map, $p^{\text{img}}$ and $p^{\text{geo}}$, which are fused via the proposed \textit{“Match-then-Fuse”} probabilistic module into a unified posterior $p^{\text{fuse}}$, from which the final rigid transformation is estimated.}
    \label{fig:method_scheme}
\end{figure*}
\vspace{-5pt}
\section{Related Work}\label{sec:related_work}
\paragraph{Hand-crafted Registration Methods.}
Early point cloud registration pipelines relied on handcrafted local descriptors such as FPFH \cite{fpfh2009} and SHOT \cite{shot2010}. These descriptors encode neighborhood geometry through histograms of normals or curvature and are typically matched via RANSAC \cite{ransac1981} and refined with ICP \cite{icp}. While effective for small-scale rigid alignment, such handcrafted features are highly sensitive to sampling density, noise, and partial overlap, limiting their robustness in complex real-world environments.
\vspace{-10pt}
\paragraph{Learning-based Registration Methods.}
With the rise of deep learning, data-driven 3D feature extractors have replaced handcrafted descriptors, providing stronger invariance and generalization. FCGF \cite{fcgf2019} introduced fully convolutional geometric features trained with contrastive learning for dense correspondence. Predator \cite{predetor2021} enhanced robustness under low overlap by combining overlap prediction and attentive feature refinement. GeoTransformer \cite{geotransformer2022} modeled spatial relations through geometric self-attention with relative positional encoding, achieving accurate alignment in cluttered scenes. RoITr \cite{roitr} embedded Point-Pair-Feature coordinates into a transformer backbone to achieve rotation invariance and distinctive geometric correspondences.

Beyond geometry-only networks, several learnable RGB-D approaches jointly utilize color and geometry to improve matching robustness. PointMBF \cite{yuan2023pointmbf} introduced a learnable dual-branch architecture that performs multi-scale bidirectional fusion between RGB and depth features through mutual attention. Unsupervised R\&R \cite{uur} trains an end-to-end RGB-D registration network without pose labels by enforcing photometric and geometric consistency using differentiable rendering. ColorPCR \cite{colorpcr} jointly learns color and geometry via a hierarchical color-enhanced feature extractor and a geo-color superpoint matching module, demonstrating improved robustness in challenging scenarios with limited geometric distinctiveness.
While these RGB-D methods demonstrate the potential of multimodal learning, they all rely on real RGB inputs and task-specific training, restricting their applicability in scenarios where the observations are point clouds only.
\vspace{-10pt}

\paragraph{Generative Based Registration Methods.}
Diffusion-based generative models, such as Stable Diffusion (SD) \cite{stablediffusion} and ControlNet \cite{controlnet}, have shown that conditioning on structural cues, particularly depth maps, enables geometry-aligned and semantically consistent image generation. This capability has recently inspired works that exploit generative priors to improve geometric alignment in registration tasks. GPCR \cite{gpcr} was the first to introduce the paradigm of leveraging geometry-to-image generation for geometry-only point cloud registration. GPCR aligns depth-map-based point clouds by synthesizing RGB images from depth inputs through a diffusion-based generative model that is fine-tuned to enforce cross-view consistency. In contrast, our method provides an alternative route for incorporating generative priors to point cloud registration. Rather than enforcing multi-view consistency through fine-tuning, we leverage a pretrained WFM that provides multi-view–consistent image generation out of the box, enabling a zero-shot operating regime with improved generalization across datasets and sensing modalities. In addition, our framework extracts correspondences from the generated images using a task-specific VFM designed for dense geometric matching. Crucially, we introduce a probabilistic \textit{match-then-fuse} formulation that combines geometric and image-derived correspondences while preserving modality-specific inductive biases, in contrast to early feature fusion strategies such as those used in GPCR.

ZeroMatch \cite{zeromatch} and FreeReg \cite{wang2024freereg} also explore the use of generative models in registration, but under different formulations. ZeroMatch enhances real RGB inputs using SD and leverages diffusion-based features to perform RGB-D cross-view registration. FreeReg tackles the RGB-to-depth registration problem by employing a generative model to bridge the modality gap between color and geometry. While both methods demonstrate the applicability of diffusion-based priors, they rely on real RGB observations and address tasks distinct from our geometry-only point cloud registration; nonetheless, they are noteworthy examples of generative approaches in this broader context.
\vspace{-5pt}
\section{Method}\label{sec:method}
\begin{figure}[t]
    \centering
    \includegraphics[
        page=2,
        width=\linewidth,
        trim={10mm 0mm 60mm 0mm},
        clip
    ]{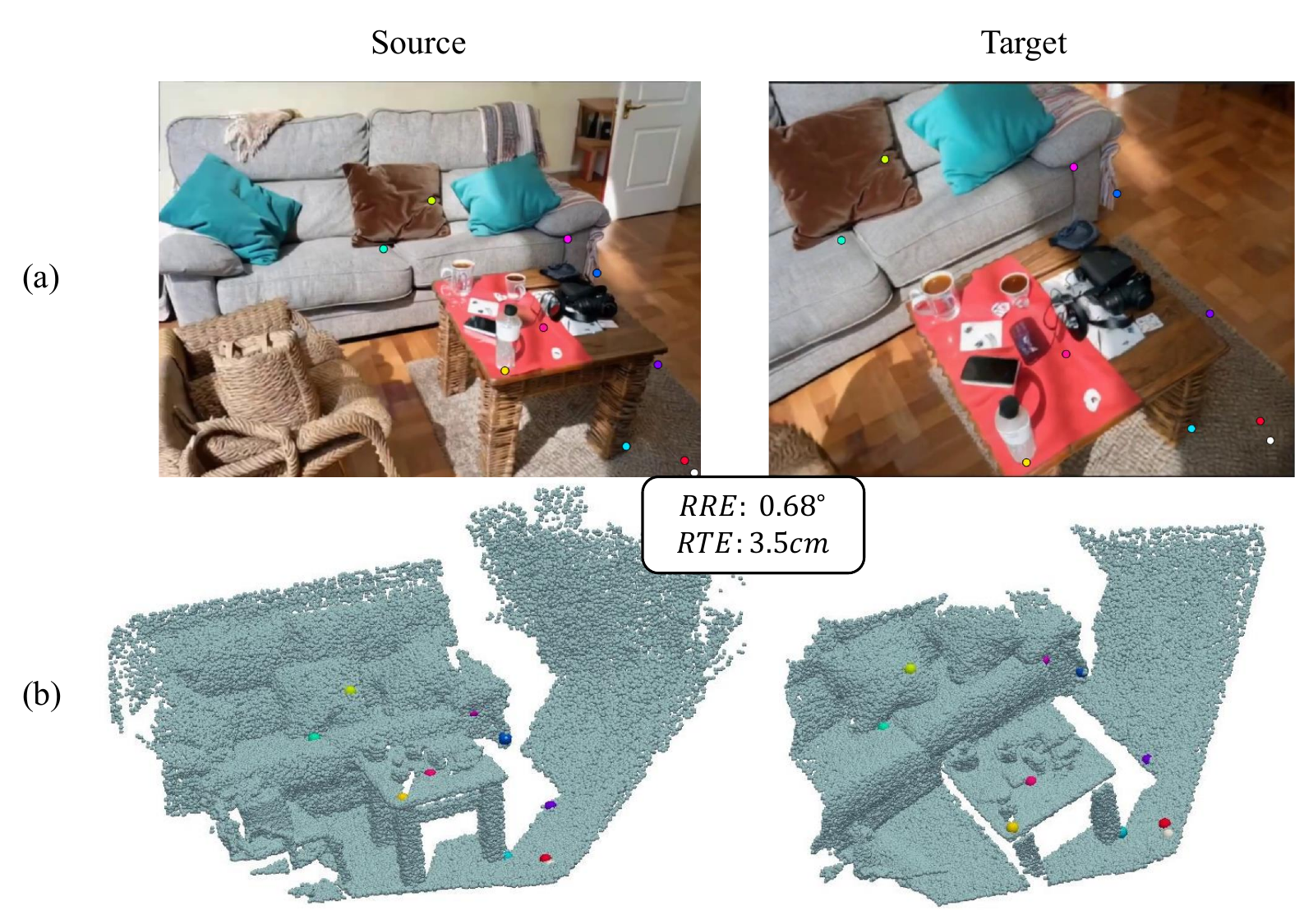}
    \caption{
        \textbf{C-GenReg qualitative example on  3DMatch.}
        Generated source and target images with a subset of matched points (color-coded correspondences), and the corresponding matches visualized on the input point clouds. The resulting rotation (RRE) and translation (RTE) errors are reported.
    }
    \label{fig:qual_matches_3dmatch_example2}
    \vspace{-10pt}
\end{figure}
\vspace{-5pt}
\subsection{Problem Definition}
Given a source point cloud $\cP \in \mathbb{R}^{N \times 3}$ and a target point cloud $\cQ \in \mathbb{R}^{M \times 3}$, the goal of point cloud registration is to estimate a rigid transformation $(\bsR, \bst) \in \mathrm{SE}(3)$ that aligns $\cP$ to $\cQ$.
If the ground-truth correspondences $\cC^* = \{(\bsp_i^*, \bsq_i^*) \mid \bsp_i^* \in \cP, \bsq_i^* \in \cQ\}$ were known, the optimal transformation is obtained by solving:
\vspace{-5pt}
\begin{equation}\label{eq:reg_def}
    \argmin_{\substack{(\bsR, \bst) \in \mathrm{SE}(3)}}
    \sum_{(\bsp^*, \bsq^*) \in \cC^*} \|\bsR \bsp_i^* + \bst - \bsq_i^*\|_2^2
    \vspace{-7pt}
\end{equation}
which has a closed form solution \cite{horn1987closed}. However, $\cC^*$ is unknown in practice, and the core challenge is to establish reliable correspondences between $\cP$ and $\cQ$. Most learning-based methods address this by extracting discriminative point-wise feature descriptors and match point pairs based on feature similarity.
\vspace{-3pt}
\subsection{C-GenReg - Overview}
\begin{figure}
    \centering
    \includegraphics[width=0.9\linewidth]{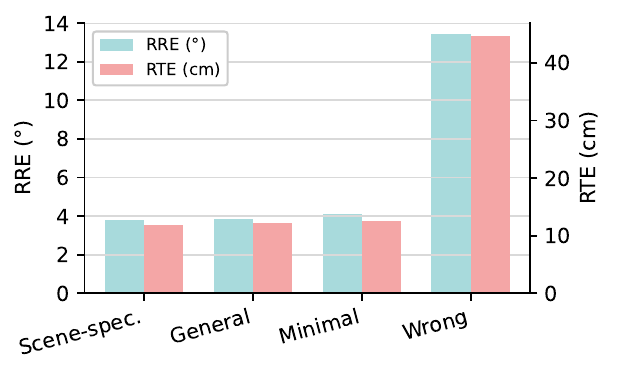}
    \vspace{-10pt}
    \caption{
\textbf{Prompt robustness on 3DMatch.}
Relative rotation (RRE,$^\circ$) and translation (RTE, cm) errors under different prompt types.
}
    \vspace{-15pt}
    \label{fig:prompt_robustness}
\end{figure}
\textit{C-GenReg} extracts complementary features for point cloud registration through a dual-branch architecture followed by a probabilistic fusion stage (\cref{fig:method_scheme}). From each input point cloud, we render a depth map and use the Cosmos-Transfer WFM \cite{cosmos_transfer} to generate multi-view-consistent RGB images that preserve geometric coherence across viewpoints. A task-specific VFM pretrained for dense geometric matching then extracts 2D features from these synthesized views, which are lifted back to 3D using the original depth to obtain per-point descriptors. In parallel, the geometric branch encodes the raw point clouds using a pretrained registration-oriented 3D feature extractor, yielding complementary geometric descriptors.

To integrate the two modalities, we adopt a ``\textit{Match-then-Fuse}" probabilistic strategy. Each branch produces a point-wise similarity matrix that is converted into a correspondence posterior, and the two posteriors are fused through a joint probabilistic model to obtain a single calibrated correspondence distribution. Putative matches sampled from this fused posterior are finally used to estimate the rigid transformation $(\bsR,\bst)$ via \cref{eq:reg_def}. The following subsections detail each component of the pipeline.

\subsection{Generated-RGB Branch}
\paragraph{World Foundation Model Consistent Generation.}
To transfer the original point-cloud geometric representation to the visual domain, we employ a WFM to generate synthetic RGB images from depth inputs. These generated images must satisfy two key properties: (1) they should be geometrically coherent, i.e., consistent with the input depth geometry, and (2) cross-view consistent, such that images rendered from different viewpoints of the same scene exhibit mutually coherent appearance.
To meet these requirements, we utilize Cosmos-Transfer \cite{cosmos_wfm, cosmos_transfer}, a WFM with strong world priors. Cosmos-Transfer supports controllable world generation from multiple modalities (e.g., segmentation, edges, depth) and is particularly effective in producing multi-view consistent RGB videos from depth control signals.

In common 3D datasets such as 3DMatch and ScanNet \cite{3dmatch2017, dai2017scannet}, point clouds are constructed by aggregating a temporal sequence of $L$ depth frames $\{D\}_{l=1}^L$ into a single voxel grid or TSDF volume \cite{tsdf}. We use this temporal depth sequence as the conditioning signal for Cosmos-Transfer. When the data is provided as LiDAR point clouds, we simulate the same input format by mounting a virtual camera and projecting the 3D points onto it to obtain depth maps (see \cref{sec:method_eval} for details).

Cosmos-Transfer expects a depth video input; therefore, we concatenate the source and target depth sequences temporally to form a single depth video. Additional details and analysis of this design choice are provided in Appendix \ref{sec:supp_temporal_concat}.

\vspace{-10pt}
\paragraph{Prompt Guidance.}
To ensure coherent and controllable generation, we use prompt-based text guidance with a fixed structure: a shared prefix that instructs the model to interpret the concatenated input as two correlated sequences while enforcing photorealism and multi-view consistency, followed by a short scene description that provides semantic context. We evaluate prompt robustness using four prompt categories that differ only in the scene description: (1) \textit{scene-specific}: ``modern home kitchen with red cabinets and a wooden dining table'', (2) \textit{general}: ``a kitchen'', (3) \textit{minimal}: ``indoor scene'', and (4) \textit{semantically wrong}: ``snowy forest''. As shown in \cref{fig:prompt_robustness}, replacing a detailed scene description with a general one results in negligible degradation, while even a minimal prompt maintains reasonably strong performance; in contrast, a semantically incorrect prompt substantially degrades registration accuracy. These results indicate that detailed semantic information is not required for successful registration, as coarse or minimal scene descriptions already provide sufficient guidance. Importantly, the availability of such coarse scene context (e.g., indoor/outdoor, street, office, laboratory) is natural in most registration scenarios, where it can typically be inferred from metadata or the acquisition environment. The prompt therefore mainly acts as a lightweight semantic stabilizer for the generative process while preserving geometric fidelity and cross-view coherence.

\vspace{-10pt}
\paragraph{Task-Specific Vision Foundation Model.}
To extract discriminative features from the generated RGB images, we employ a task-specific VFM tailored for image matching and registration. Specifically, we use MASt3R \cite{mast3r_eccv24}, a VFM trained to produce dense correspondence-aware features. This choice is motivated by the inductive bias and feature structure of task-oriented VFMs, which better aligns with the objectives of geometric registration compared to general-purpose vision models. We further validate this design in our ablation study (\cref{sec:ablation_study}), where task-specific VFMs yield a clear performance gain over general-purpose alternatives.
\vspace{-10pt}

\paragraph{View Selection.}
MASt3R operates on pairs of source and target images through a cross-attention–based decoder, where the extracted features for each image are conditioned on the paired counterpart. Consequently, a given source image produces different feature maps when evaluated with different target images. To exploit this property, we sample $K$ views from each domain and evaluate all pairwise combinations, resulting in $K^2$ conditioned feature maps per domain. While increasing $K$ improves viewpoint coverage, it also increases computational cost. Since the original $L$ frames within each sequence are highly correlated, we find that relatively small values of $K$ provide sufficient representational diversity. In practice, we therefore set $K \ll L$ to balance efficiency and performance. Additional analysis of the view selection strategy is provided in Appendix \ref{sec:view_selection}.

\vspace{-10pt}
\paragraph{2D to 3D Lifting.}
Since the generated RGB frames originate from depth inputs, we can lift the 2D image features back to 3D space using the known depth camera intrinsics (or virtual intrinsics for LiDAR data). Each 3D point is thus assigned the feature of its corresponding image pixel. In practice, the dense image features $(HW)$ greatly exceed the number of points used in the geometric branch $(N_{\text{src}},N_{\text{tgt}})$, since the latter are typically voxel-downsampled. To align both modalities, we perform a nearest-neighbor query from the observed point clouds to the dense lifted points, assigning each  point in the point cloud the feature of its closest image-based neighbor. This produces feature representations of consistent size for both modalities, $F^{\text{img}}_{n}\in \mathbb{R}^{K^2 \times N_{n} \times d_{\text{img}}}$ and $F^{\text{geo}}_{n}\in \mathbb{R}^{N_{n} \times d_{\text{geo}}}$, where $n\in\{\text{src}, \text{tgt}\}$.

\subsection{Geometric Branch}
The parallel geometric branch  directly processes the input point clouds. Its goal is to extract geometry-oriented features that may not be fully captured by the image-based representation. For this purpose, we employ a frozen point cloud feature extractor pretrained for matching and registration tasks. Multiple architectures can serve as the geometric backbone, as long as they provide dense per-point feature representations for both source and target clouds, defined as \( F^{\text{geo}}_{\text{src}} \in \mathbb{R}^{N_{\text{src}} \times d_{\text{geo}}} \) and \( F^{\text{geo}}_{\text{tgt}} \in \mathbb{R}^{N_{\text{tgt}} \times d_{\text{geo}}} \). In our ablation study (\cref{sec:ablation_study}), we evaluate several candidates and find that GeoTransformer \cite{geotransformer2022} yields the best performance results; therefore, we adopt it in our final design.

\subsection{Match-then-Fuse Probabilistic Fusion}
The fusion module is designed with two main objectives: (1) to preserve the inductive biases of the pretrained feature extractors, which are optimized for point matching- each for its own domain, and (2) to remain entirely non-trainable, keeping the overall framework training-free. To meet these goals, we propose a \emph{“match-then-fuse”} probabilistic strategy, where putative correspondences are first established independently for each modality by computing feature similarity matrices between source and target points. Each similarity matrix is then converted into a posterior distribution using a row-wise softmax operation, and the resulting modality-specific posteriors are fused in a probabilistic manner to produce a unified correspondence matrix. This fused probability map serves as the basis for estimating the final point-to-point matches and the rigid transformation.

\begin{table*}
\centering
\setlength{\tabcolsep}{2pt} % tighter column spacing
\renewcommand{\arraystretch}{1.05} % slightly more compact rows
\resizebox{0.7\textwidth}{!}{%
\begin{tabular}{@{}p{2.1cm} l
                c c c c c
                c c c c c@{}}
\Xhline{1.3pt}
\multicolumn{2}{c}{} &
\multicolumn{5}{c}{\textbf{Rotation} $[deg]$} &
\multicolumn{5}{c}{\textbf{Translation} $[cm]$} \\
\cmidrule(lr){3-7} \cmidrule(l){8-12}
\multicolumn{2}{c}{} &
\multicolumn{3}{c}{\textbf{Accuracy} $\uparrow$} &
\multicolumn{2}{c}{\textbf{Error} $\downarrow$} &
\multicolumn{3}{c}{\textbf{Accuracy} $\uparrow$} &
\multicolumn{2}{c}{\textbf{Error} $\downarrow$} \\
\cmidrule(lr){3-5} \cmidrule(lr){6-7} \cmidrule(lr){8-10} \cmidrule(l){11-12}
\arraybackslash \textbf{Type} & \textbf{Method} &
$5$ & $10$ & $45$ &
Mean & Med. &
$5$ & $10$ & $25$ &
Mean & Med. \\
\Xhline{1.3pt}
\multirow{1}{*}{\centering Hand-Crafted} 
  & FPFH \cite{fpfh2009}
  & 41.4 & 56.7 & 73.3 & 39.2 & 7.1  & 17.5 & 35.1 & 50.9 & 79.5 & 23.5 \\
\Xhline{0.9pt}
\multirow{6}{*}{\centering\makecell{Learned\\(PC Only)}} 
  & GeoTrans. \cite{geotransformer2022}
  & 88.9 & 91.8 & 93.3 & 12.0 & \underline{1.4} & \textbf{59.8} & 81.0 & 90.1 & 24.6 & \textbf{4.0} \\
  & FCGF \cite{fcgf2019}
  & 90.4 & 93.7 & 94.8 & 9.4 & \underline{1.4} & 53.4 & 79.3 & 91.0 & 19.2 & 4.7 \\
  & Predator \cite{predetor2021}
  & 85.0 & 91.5 & 94.2 & 10.5 & 2.0 & 42.1 & 72.5 & 87.1 & 22.6 & 5.8 \\
  & RoITr \cite{roitr}
  & 86.3 & 91.1 & 93.8 & 11.1 & 1.6 & 51.2 & 77.4 & 89.1 & 20.5 & 4.9  \\
  & GPCR \cite{gpcr}
  & \textbf{94.3} & \underline{96.7} & \underline{98.1} & \underline{4.5} & \underline{1.4} & 54.3 & \underline{81.5} & \underline{93.1} & \underline{12.5} & 4.7 \\
  \cdashline{2-12}[0.8pt/2pt]
  & \textbf{C-GenReg (Ours)}
  & \underline{94.2} & \textbf{97.5} & \textbf{98.3} & \textbf{3.8} & \textbf{1.3} & \underline{57.5} & \textbf{82.0} & \textbf{95.7} & \textbf{11.9} & \underline{4.3} \\
\Xhline{0.9pt}
\multirow{3}{*}{\centering\makecell{Learned\\(RGB-D)}} 
  & PointMBF \cite{yuan2023pointmbf}
  & 80.9 & 86.4 & 92.4 & 12.1 & 1.6 & 52.2 & 73.6 & 85.1 & 24.5 & 4.7 \\
  & ZeroMatch \cite{zeromatch}
  & 93.5 & 97.1 & 98.4 & 3.6 & 1.4 & 52.8 & 81.0 & 94.1 & 10.8 & 4.7 \\
  & C-GenReg (Oracle)
  & 95.1 & 99.6 & 99.8 & 2.1 & 1.4 & 62.2 & 84.1 & 98.3 & 7.3 & 3.8 \\
\Xhline{1.3pt}
\end{tabular}%
}
\caption{\textbf{3DMatch Benchmark.} Rotation and translation accuracy (\% of pairs within RRE/RTE thresholds in $deg$ and $cm$ respectively) and mean/median error across different methods. RGB-D baselines are included as complementary reference. Best in \textbf{bold}, second-best in \underline{underlined}.}
\label{tab:3dmatch_res}
\end{table*}
\vspace{-10pt}
\paragraph{Modality Correspondence Posterior.}
Let $M_{ij}$ be a binary random variable indicating whether source point $i$ corresponds to target point $j$ and $S_{ij}$ their features similarity.
To approximate the modality-specific correspondence posterior
$ \Pr(M_{ij} | S_{ij}^{m}) $, where $ m \in \{\mathrm{geo}, \mathrm{img}\} $,
we first compute the source-target feature similarity matrices for each modality and then apply a row-wise softmax normalization.
\vspace{-8pt}
\begin{align}
S^{\mathrm{geo}} &= F^{\mathrm{geo}}_{\mathrm{src}} (F^{\mathrm{geo}}_{\mathrm{tgt}})^{\top}, \\
S^{\mathrm{img}} &= \max_{k \in \{1,\dots,K^2\}} 
    F^{\mathrm{img}}_{\mathrm{src},k} (F^{\mathrm{img}}_{\mathrm{tgt},k})^{\top}.
\end{align}
The feature vectors are $\ell_2$-normalized, and for the image branch, the similarity is computed for all $K^2$ view-pair combinations; 
for each point pair $(i,j)$, we retain the maximal similarity across the view dimension, 
capturing the best cross-view match. 
Finally, the modality correspondence posterior is obtained as:
\vspace{-5pt}
\begin{equation}\label{eq:modality_posterior}
p_{ij}^{m} \triangleq \Pr(M_{ij}=1 \mid S_{ij}^{m})
= \operatorname{Softmax}_{j} \big(S_{ij}^{m} / \tau_{m}\big),
\vspace{-5pt}
\end{equation}
where $\tau_{m}$ is a temperature parameter.
\vspace{-5pt}
\paragraph{Joint Posterior Fusion \emph{(Noisy-AND)}.}
We propose a Joint Posterior Fusion that combines the modality correspondence posteriors from the image and geometric branches:
\vspace{-5pt}
\begin{equation}
p^{\text{fuse}}_{ij}\triangleq \Pr( M_{ij}=1\mid S_{ij}^{img}, S_{ij}^{geo})   
\vspace{-3pt}
\end{equation}
The fused posterior favors correspondences jointly supported by both modalities and thus acts as a probabilistic \emph{Noisy-AND}, where mutual agreement increases confidence.
Each branch operates independently using frozen pretrained models with distinct modality priors, WFM + VFM for texture and Geometric Feature Extractor for geometry. Although the RGB features originate from depth, the two branches process information in fundamentally different domains; conditioned on the true correspondence $M_{ij}$, their remaining dependence is negligible. We therefore assume conditional independence:
$   S_{ij}^{img}\mathrel{\perp\mspace{-10mu}\perp} S_{ij}^{geo} \mid M_{ij}, ~ \forall(i,j).$ 
Thus,  the joint posterior fusion probability is given by
\vspace{-10pt}
\begin{equation}\label{eq:P_and_fuse}
p^{\text{fuse}}_{ij}
=\frac{p^{\text{img}}_{ij}\,p^{\text{geo}}_{ij}\,(1-\pi_{ij})}
{p^{\text{img}}_{ij}\,p^{\text{geo}}_{ij}\,(1-\pi_{ij})
+\big(1-p^{\text{img}}_{ij}\big)\big(1-p^{\text{geo}}_{ij}\big)\,\pi_{ij}},
\vspace{-5pt}
\end{equation}
where $\pi_{ij} \triangleq  \Pr(M_{ij}=1)$ is the prior matching probability (derivation of \cref{eq:P_and_fuse} is provided in Appendix \ref{sec:nosiy_and_prof}).
\vspace{-15pt}
\paragraph{Disjunctive Posterior Fusion \textit{(Noisy-OR)}.}
The Joint Posterior Fusion \textit{(Noisy-AND)} favors correspondences jointly supported by both modalities but may overlook matches strongly indicated by only one. To address this, we introduce the Disjunctive Posterior Fusion \textit{(Noisy-OR)}, which aggregates evidence in a union-like manner, increasing correspondence confidence when either the image or geometric branch provides strong support.
Let $A_{ij}^{m}$ be a modality activation random variable indicating whether modality $m$ supports the correspondence $(i,j)$. 
We model $A_{ij}^{m}$ as a Bernoulli random variable depending only on its own similarity signal:
$   A_{ij}^{m} \mid S_{ij}^{m} \sim \mathrm{Bernoulli}\!\big(p_{ij}^{m}\big),
    \Pr(A_{ij}^{m}=1 \mid S_{ij}^{m}) = p_{ij}^{m}.$
Under the conditional independence of activations given their respective modality signals, 
the \emph{Noisy-OR} fusion is given by (see Appendix \ref{sec:nosiy_or_prof} for full derivation):
\begin{equation}\label{eq:p_noisy_or_final}
    p^{\text{Noisy-OR}}_{ij}
    = 1 - (1 - p^{\text{img}}_{ij})(1 - p^{\text{geo}}_{ij}),
\end{equation}
Ablation studies (\cref{sec:ablation_study}) compare the two fusion paradigms, with the Noisy-AND ultimately selected as our design choice.
\vspace{-7pt}
\paragraph{Transformation Estimation.}
To recover the rigid transformation, we first derive the correspondence set $\cC$ from the fused posterior matrix $p_{ij}^{\text{fuse}}$ by applying a mutual nearest-neighbor matching strategy. Since no prior information about the correspondence distribution is available, we assume a uniform prior $\pi_{ij}=\frac{1}{N_{\text{src}}N_{\text{tgt}}}$ . The final transformation parameters are then estimated by solving \cref{eq:reg_def}. In line with common practice, we employ a robust estimator to reduce the influence of outliers, using SC2PCR \cite{sc2pcr} similarly to \cite{gpcr}.
\begin{table*}[t]
\centering
\setlength{\tabcolsep}{3pt} % slightly tighter columns
\renewcommand{\arraystretch}{1.05}
\resizebox{\textwidth}{!}{%
\begin{tabular}{@{}l
                c c c c c c c c c c
                !{\vrule width 1.5pt}
                c c c c c c c c c c@{}}
\Xhline{1.5pt}
\multicolumn{1}{c}{} &
\multicolumn{10}{c}{\textbf{ScanNet Hard}} &
\multicolumn{10}{c}{\textbf{ScanNet SuperGlue Split}} \\
\cmidrule(lr){2-11} \cmidrule(l){12-21}
\textbf{Method} &
\multicolumn{5}{c}{\textbf{Rotation} $[deg]$} &
\multicolumn{5}{c}{\textbf{Translation} $[cm]$} &
\multicolumn{5}{c}{\textbf{Rotation} $[deg]$} &
\multicolumn{5}{c}{\textbf{Translation} $[cm]$} \\
\cmidrule(lr){2-6} \cmidrule(lr){7-11} \cmidrule(lr){12-16} \cmidrule(l){17-21}
& $5$ & $10$ & $45$ & Mean & Med. &
$5$ & $10$ & $25$ & Mean & Med. &
$5$ & $10$ & $45$ & Mean & Med. &
$5$ & $10$ & $25$ & Mean & Med. \\
\Xhline{1.5pt}
GeoTrans. \cite{geotransformer2022}
  & 71.5 & 78.0 & 83.4 & 26.2 & 2.0 & 48.4 & 65.2 & 74.6 & 51.9 & 5.2
  & 74.0 & 80.2 & 85.9 & 21.9 & 1.5 & \underline{54.2} & 67.3 & 77.1 & 72.6 & 4.2 \\
FCGF \cite{fcgf2019}
  & 78.9 & 84.2 & 87.5 & 19.4 & \textbf{1.5} & 55.3 & 70.7 & 79.7 & 37.8 & 4.3 & 79.1 & 86.2 & 90.9 & 13.0 & 1.9 & 42.8 & 68.8 & 82.1 & 38.1 & 5.9\\
Predator \cite{predetor2021}
  & 64.3 & 75.2 & 82.6 & 26.3 & 3.2 & 30.1 & 54.8 & 69.2 & 48.7 & 8.4
  & 82.0 & 88.7 & 92.2 & 12.5 & 1.8 & 47.5 & 71.6 & 85.9 & 41.9 & 5.2 \\
RoITr \cite{roitr}
  & 70.0 & 77.5 & 83.7 & 24.1 & 2.3 & 40.3 & 62.3 & 75.1 & 45.6 & 6.5
  & \underline{88.4} & \underline{91.2} & \underline{93.2} & \underline{11.1} & \textbf{1.2} & \textbf{64.8} & \underline{83.0} & \underline{89.1} & \underline{33.8} & \underline{3.4} \\
GPCR \cite{gpcr}
  & \underline{82.9} & \underline{90.0} & \underline{94.4} & \underline{8.4} & \underline{1.6} & \underline{56.4} & \underline{73.0} & \underline{82.7} & \textbf{21.7} & \underline{4.1}
  & \multicolumn{10}{c}{\textit{No results reported}} \\
\Xhline{1pt}
\textbf{C-GenReg}
  & \textbf{88.7} & \textbf{92.9} & \textbf{94.9} & \textbf{7.8} & 1.7
  & \textbf{61.8} & \textbf{79.8} & \textbf{88.1} & \underline{23.0} & \textbf{3.3}
  & \textbf{89.5} & \textbf{92.0} & \textbf{94.6} & \textbf{8.4} & \underline{1.3}
  & \textbf{64.8} & \textbf{83.2} & \textbf{89.6} & \textbf{32.2} & \textbf{3.0} \\
\Xhline{1.5pt}
\end{tabular}%
}
\caption{\textbf{ScanNet Benchmarks.} Rotation and translation accuracy (\% of pairs within RRE/RTE thresholds in $deg$ and $cm$ respectively) and mean/median error on the ScanNet Hard and ScanNet SuperGlue Split benchmarks. Best results are in \textbf{bold}, second-best \underline{underlined}.}
\label{tab:scannet_results}
\end{table*}
\section{Experiments}\label{sec:results}

\begin{table}
\centering
\setlength{\tabcolsep}{1.8pt} % compact horizontal spacing
\renewcommand{\arraystretch}{0.98} % compact vertical spacing
\resizebox{\columnwidth}{!}{%
\begin{tabular}{@{}l
                c c c c c
                c c c c c@{}}
\Xhline{1.3pt}
\multicolumn{1}{c}{} &
\multicolumn{5}{c}{\textbf{Rotation} [$deg$]} &
\multicolumn{5}{c}{\textbf{Translation} [$m$]} \\
\cmidrule(lr){2-6} \cmidrule(l){7-11}
\textbf{Method} &
\multicolumn{3}{c}{\textbf{Accuracy} $\uparrow$} &
\multicolumn{2}{c}{\textbf{Error} $\downarrow$} &
\multicolumn{3}{c}{\textbf{Accuracy} $\uparrow$} &
\multicolumn{2}{c}{\textbf{Error} $\downarrow$} \\
\cmidrule(lr){2-4} \cmidrule(lr){5-6} \cmidrule(lr){7-9} \cmidrule(l){10-11}
& $1$ & $2$ & $5$ & Mean & Med. &
$0.6$ & $1$ & $2$ & Mean & Med. \\
\Xhline{1.3pt}
GeoTrans.
  & 17.0 & 39.6 & 80.8 & 7.3 & 2.5
  & 2.2 & 18.4 & 59.5 & 4.1 & 1.6 \\
FCGF
  & 14.7 & 27.7 & 55.5 & 15.4 & 4.4
  & 2.2 & 5.1 & 26.6 & 7.4 & 10.1 \\
Predator
  & 21.0 & 49.0 & 65.1 & 10.0 & 2.0
  & 1.4 & 13.3 & 61.9 & 4.9 & 1.3 \\
\Xhline{0.9pt}
\textbf{C-GenReg}
  & \textbf{61.8} & \textbf{76.2} & \textbf{86.3} & \textbf{2.4} & \textbf{0.6}
  & \textbf{41.1} & \textbf{52.9} & \textbf{64.6} & \textbf{1.7} & \textbf{0.9} \\
\Xhline{1.3pt}
\end{tabular}%
}
\caption{\textbf{Waymo Outdoor Registration Benchmark.} Rotation ($deg$) and translation ($m$) accuracy/error. Best results are in \textbf{bold}.}
\vspace{-15pt}
\label{tab:waymo_results}
\end{table}
\vspace{-3pt}
\subsection{Experimental Settings}
\paragraph{Datasets and Benchmarks.}
We evaluate our method on two benchmark types: indoor datasets captured by depth sensors and outdoor dataset acquired by LiDAR. For indoor evaluation, we adopt the widely used 3DMatch and ScanNet benchmarks \cite{3dmatch2017, dai2017scannet}, where 3DMatch serves as the primary evaluation set and ScanNet as a cross-dataset generalization benchmark. For outdoor evaluation, we employ the Waymo Open Dataset \cite{waymo}, which contains large-scale LiDAR scans, and serves as a generalization benchmark for outdoor registration tasks. Additional benchmarks, including low-overlap evaluation, are provided in Appendix \ref{sec:addtional_res_supp}.
 \vspace{-10pt}
\paragraph{Implementation Details.}
For the \textit{C-GenReg} pipeline, we employ Cosmos-Transfer-v1 (Depth) \cite{cosmos_transfer} as the WFM and MASt3R \cite{mast3r_eccv24} as the VFM. The geometric feature extractor is based on GeoTransformer \cite{geotransformer2022}, while the probabilistic fusion module follows the Noisy-AND formulation. The feature dimensions of the respective models are $d_{img}=24$ and $d_{geo}=256$. For the VFM branch, we use $K=4$ input views from the $L=50$ views, and the probability temperature parameter is $\tau_m=0.1$. All models in the pipeline are kept frozen with their publicly released pretrained weights, without any additional fine-tuning. Additional implementation details including runtime analysis are provided in Appendices  \ref{sec:implementation_details_supp} and \ref{sec:runtime_supp}.
 \vspace{-10pt}
\paragraph{Metrics.}
We follow the standard evaluation protocol for point cloud registration \cite{gpcr, uur, yuan2023pointmbf, zeromatch}, reporting the Relative Rotation Error (RRE) and Relative Translation Error (RTE). For each benchmark, we report both the mean and median values of these errors, as well as the registration accuracy - the percentage of registration problems with an error below a given threshold.

\begin{table*}[t]
\centering
\setlength{\tabcolsep}{2.2pt}
\renewcommand{\arraystretch}{1.05}
\resizebox{0.7\textwidth}{!}{%
\begin{tabular}{@{}l c c
                c c c c c
                c c c c c@{}}
\Xhline{1.5pt}
\multicolumn{3}{c}{} &
\multicolumn{5}{c}{\textbf{Rotation} $[deg]$} &
\multicolumn{5}{c}{\textbf{Translation} $[cm]$} \\
\cmidrule(lr){4-8} \cmidrule(l){9-13}
\textbf{VFM} & \textbf{Geo. Features} & \textbf{Fusion} &
\multicolumn{3}{c}{\textbf{Accuracy} $\uparrow$} &
\multicolumn{2}{c}{\textbf{Error} $\downarrow$} &
\multicolumn{3}{c}{\textbf{Accuracy} $\uparrow$} &
\multicolumn{2}{c}{\textbf{Error} $\downarrow$} \\
\cmidrule(lr){4-6} \cmidrule(lr){7-8} \cmidrule(lr){9-11} \cmidrule(l){12-13}
& & & $5$ & $10$ & $45$ & Mean & Med. & $5$ & $10$ & $25$ & Mean & Med. \\
\Xhline{1.5pt}
\multicolumn{3}{c}{\underline{\textit{VFM Ablation}}} & & & & & & & & & &\\
DINOv2 \cite{dino2} & \multirow{3}{*}{--} & \multirow{3}{*}{--} 
  & 57.6 & 67.2 & 79.4 & 27.4 & 3.5 & 26.1 & 46.7 & 62.1 & 73.3 & 11.8 \\
RoMa \cite{roma} &  &  
  & 82.6 & \textbf{89.9} & \textbf{94.6} & \textbf{9.0} & \textbf{2.2} & 34.7 & 66.1 & 81.4 & 34.5 & 6.8 \\
MASt3R \cite{mast3r_eccv24} &  &  
  & \textbf{82.7} & 87.2 & 91.1 & 11.7 & 2.4 & \textbf{38.9} & \textbf{66.3} & \textbf{82.9} & \textbf{32.5} & \textbf{6.4} \\
\Xhline{1pt}
\multicolumn{3}{c}{\underline{\textit{Geo. Features and Fusion Ablation}}} & & & & & & & & & &\\
\multirow{3}{*}{MASt3R} & \multirow{3}{*}{FCGF} & Concat.
  & 85.2 & 91.3 & 94.6 & 9.3 & 1.7 & 37.3 & 69.6 & 88.1 & 26.2 & 6.4 \\
&  & Noisy OR
  & 90.6 & 95.8 & 97.1 & 5.2 & 1.4 & 45.0 & 76.5 & 93.0 & 16.1 & 5.4 \\
&  & Noisy AND
  & 91.2 & 96.0 & 97.1 & 5.1 & 1.5 & 55.0 & 80.4 & 93.1 & 16.2 & 4.6 \\
\Xhline{0.9pt}
\multirow{3}{*}{MASt3R} & \multirow{3}{*}{Predator} & Concat.
  & 80.6 & 86.4 & 89.9 & 14.1 & 1.8 & 35.9 & 64.9 & 83.8 & 37.9 & 6.7 \\
&  & Noisy OR
  & 88.1 & 92.7 & 94.9 & 8.9 & 1.6 & 40.0 & 70.8 & 90.8 & 22.0 & 5.9 \\
&  & Noisy AND
  & 88.6 & 92.8 & 95.1 & 8.8 & 1.6 & 42.9 & 72.5 & 90.9 & 19.5 & 4.8 \\
\Xhline{0.9pt}
\multirow{3}{*}{MASt3R} & \multirow{3}{*}{GeoTrans.} & Concat.
  & 79.4 & 83.4 & 86.1 & 21.9 & 1.6 & 46.5 & 70.1 & 81.2 & 60.1 & 5.5 \\
&  & Noisy OR
  & 94.2 & 97.4 & 98.0 & 3.9 & 1.3 & 51.3 & 82.0 & 95.7 & 12.1 & 4.8 \\
&  & Noisy AND
  & \textbf{94.2} & \textbf{97.5} & \textbf{98.3} & \textbf{3.8} & \textbf{1.3}
  & \textbf{57.5} & \textbf{82.1} & \textbf{95.7} & \textbf{11.9} & \textbf{4.3} \\
\Xhline{1.5pt}
\end{tabular}%
}
\caption{\textbf{Ablation Study on the 3DMatch Benchmark.} Top: impact of different Vision Foundation Models (no geometric features or fusion).
Bottom: impact of geometric feature extractors and fusion operators (using MASt3R as the VFM). Best in \textbf{bold}.}
\label{tab:ablations}
\end{table*}
\subsection{Method Evaluation}\label{sec:method_eval}

\paragraph{3DMatch Benchmark (Indoor).}
We begin by evaluating our method on the widely used 3DMatch benchmark (\cref{tab:3dmatch_res}). C-GenReg is compared against both the hand-crafted descriptor FPFH \cite{fpfh2009} and several state-of-the-art (SOTA) learning-based baselines, including GeoTransformer \cite{geotransformer2022}, FCGF \cite{fcgf2019}, Predator \cite{predetor2021}, RoITr \cite{roitr}, and GPCR \cite{gpcr}. All learning-based methods are trained on the official 3DMatch training split.
Despite this, C-GenReg achieves the best overall performance across most rotation and translation metrics. It reduces the mean RTE by nearly half compared to GeoTransformer and consistently achieves superior rotation accuracy, demonstrating the benefit of our probabilistic fusion. While GeoTransformer attains slightly higher translation accuracy at the strict $5cm$ threshold and a marginally lower median RTE, and GPCR reports a $0.1 pp$ advantage in rotation accuracy at $5^\circ$, C-GenReg surpasses both in the majority of metrics and achieves the lowest overall errors. A qualitative example of the correspondences produced by C-GenReg is shown in \cref{fig:qual_matches_3dmatch_example2}, with further examples available in Appendix \ref{sec:qualitative_supp}.

As a reference, we also compare C-GenReg with RGB-D based registration methods ZeroMatch \cite{zeromatch} and PointMBF \cite{yuan2023pointmbf} that use \textit{real RGB images} as input. Although this comparison is not strictly fair, since C-GenReg relies solely on 3D point cloud inputs, it is noteworthy that C-GenReg achieves comparable results to ZeroMatch and even outperforms PointMBF. For reference, we additionally report C-GenReg-Oracle, which replaces the generated RGB with the real RGB input to provide an empirical upper bound on our pipeline’s potential performance.

\vspace{-10pt}
\paragraph{ScanNet Benchmarks (Indoor).}
To evaluate cross-dataset generalization, we benchmark all methods on the ScanNet indoor registration benchmarks (\cref{tab:scannet_results}). All learning-based approaches are trained on 3DMatch, while evaluation is performed on unseen ScanNet data. We follow the ScanNet Hard protocol introduced in \cite{gpcr,zeromatch}, where source and target frames are 50 frames apart, resulting in significantly lower overlap. As the original pairs were not released, we reconstructed the benchmark following the authors’ instructions. On this split, C-GenReg achieves the best overall performance across most metrics, demonstrating strong generalization to unseen environments. In particular, it ranks first or second on most of the metrics, with GPCR slightly outperforming in median RRE and mean RTE.

For reproducibility, we also report results on the ScanNet SuperGlue split \cite{superglue,yuan2023pointmbf}, which provides official registration pairs and poses a more challenging setting than both the original and Hard splits.
Since GPCR code is unavailable, it is excluded from this comparison. C-GenReg again attains the best or second-best results in nearly all metrics, while RoITr achieves a marginally lower median RRE. Results for the original ScanNet benchmark are given in Appendix \ref{sec:scannet_org_supp}.
\vspace{-10pt}
\paragraph{Waymo Benchmark (Outdoor).}
To assess cross-dataset generalization on outdoor LiDAR data, we evaluate on the Waymo dataset \cite{waymo} (\cref{tab:waymo_results}).
We sample 1,500 registration pairs from the validation split, selecting frame pairs at least 50 frames apart and within $30 m$ based on ground-truth ego motion.
As described in \cref{sec:method}, C-GenReg operates on depth maps; to adapt it for LiDAR, each scan is projected onto a virtual camera to generate a synthetic depth image, after which the same pipeline is applied.
In this experiment, a single forward-facing virtual camera is used (details in Appendix \ref{sec:lidar_supp}).

We compare against learning-based registration methods trained on KITTI \cite{kitti}.
As shown in \cref{tab:waymo_results}, these methods generalize poorly to Waymo due to sensor discrepancies (different beam patterns and densities), whereas C-GenReg achieves substantially better rotation and translation accuracy.
Example generated images from the WFM are shown in \cref{fig:teaser}, with additional qualitative results provided in Appendix \ref{sec:qualitative_supp}.

\subsection{Ablation Studies}\label{sec:ablation_study}
We perform an extensive ablation studies to analyze the contribution of each component in the C-GenReg pipeline. All experiments are conducted on the 3DMatch dataset, and the results are summarized in \cref{tab:ablations}.
\vspace{-15pt}
\paragraph{General Purpose vs. Task Specific VFM.}
We investigate the impact of using a general-purpose vs. a task-specific VFM. As a general model, we use DINOv2 \cite{dino2}, trained in a self-supervised manner on large-scale generic image data. For task-specific VFMs, we evaluate MASt3R \cite{mast3r_eccv24}, trained explicitly for image matching, and RoMa \cite{roma}, a DINO-based model fine-tuned for image registration. In this ablation, the geometric branch is bypassed, and only the generated-image branch is used. As shown in \cref{tab:ablations}, task-specific VFMs dramatically outperform the general-purpose one, achieving roughly $2\times$ lower mean RTE and up to $3\times$ lower mean RRE, highlighting the importance of task alignment. Between the two task-specific models, MASt3R and RoMa achieve comparable results, with MASt3R yielding slightly better translation accuracy; however, we adopt MASt3R in our pipeline due to its denser feature outputs, which integrate more effectively with the probabilistic fusion stage.
\vspace{-12pt}
\paragraph{Geometric Feature Extractor Ablation.}
We evaluate three geometric feature extractors: FCGF, Predator, and GeoTransformer as candidates for our pipeline, using MASt3R as the VFM. As shown in \cref{tab:ablations}, integrating our generated-RGB branch with each geometric backbone consistently improves performance compared to their original baselines in \cref{tab:3dmatch_res}, demonstrating that C-GenReg serves as a general performance enhancer across geometric feature extractors. Among the tested options, GeoTransformer yields the best results and is therefore adopted as our default geometric feature extractor.
\vspace{-13pt}
\paragraph{Fusion Method Ablation.}
To validate our proposed ``\textit{Match-then-Fuse}" paradigm, we compare it with the alternative ``Fuse-then-Match" strategy, which concatenates features from both modalities before matching, as done in \cite{gpcr,zeromatch}. As shown in \cref{tab:ablations}, our probabilistic fusion approach yields a clear performance gain over simple feature concatenation across all geometric backbones, with up to 5× improvement in mean RRE and RTE when using GeoTransformer features.
We further compare our two probabilistic variants: Noisy-AND and Noisy-OR. While both substantially outperform concatenation, Noisy-AND achieves slightly better registration accuracy and consistently produces higher-precision point matches. Since accurate registration depends heavily on a small set of reliable correspondences, this higher precision motivates choosing Noisy-AND as the default fusion operator in C-GenReg. Additional ablations are provided in Appendix \ref{sec:fuse_ablation_supp}.

\vspace{-5pt}
\section{Conclusions}
We presented \textit{C-GenReg}, a training-free framework for 3D point cloud registration that leverages the complementary strengths of world-scale generative priors and registration-oriented VFMs. Current learning-based 3D registration methods struggle to generalize across sensing modalities, sampling differences, and environments. Hence, \textit{C-GenReg} generates an auxiliary image channel, where VFMs excel, by employing a WFM to synthesize auxiliary multi-view-consistent RGB representations from the input geometry. This generative transfer preserves spatial coherence across source and target views without any fine-tuning. From these generated views a VFM pretrained for dense correspondences extracts matches. The resulting pixel correspondences are lifted back to 3D via the original depth maps, producing a set of geometry-aware correspondences. We further introduce a “\textit{Match-then-Fuse}” probabilistic cold-fusion scheme that combines independent correspondence posteriors from the generated-RGB and geometric branches. This procedure preserves each modality's inductive bias and provides calibrated confidence without any additional learning. \textit{C-GenReg} is fully zero-shot and plug-and-play: all modules are pretrained and operate without fine-tuning. Extensive experiments on indoor and outdoor benchmarks demonstrate strong zero-shot performance and superior cross-domain generalization.
{
    \small
    \bibliographystyle{ieeenat_fullname}
    \bibliography{main}
}


\newpage
\appendix
\maketitlesupplementary

\section{Probabilistic Fusion Derivation}
\subsection{Noisy-AND Derivation (Eq. 6)}\label{sec:nosiy_and_prof}
\begin{proposition}[Joint Posterior Fusion \emph{(Noisy-AND)}] \label{prop:P_fuse}
Under the conditional independence assumption of the image and geometric evidences 
given the match event, the Joint Posterior Fusion probability satisfies:
\vspace{-5pt}
\begin{equation} \label{eq:P_fuse}
p^{\text{fuse}}_{ij}
=\frac{p^{\text{img}}_{ij}\,p^{\text{geo}}_{ij}\,(1-\pi_{ij})}
{p^{\text{img}}_{ij}\,p^{\text{geo}}_{ij}\,(1-\pi_{ij})
+\big(1-p^{\text{img}}_{ij}\big)\big(1-p^{\text{geo}}_{ij}\big)\,\pi_{ij}}.
\vspace{-4pt}
\end{equation}
\end{proposition}
% Proff
\begin{proof}[\cref{prop:P_fuse}]
Define the odds $o(x)=\frac{x}{1-x}$.

By Bayes rule in odds form and conditional independence of $ S_{ij}^{img}\perp\!\!\!\perp S_{ij}^{geo} \mid M_{ij}$,
\begin{align}
O^{\text{fuse}}_{ij}
&=\frac{\Pr(M_{ij}{=}1\mid S_{ij}^{img}, S_{ij}^{geo})}{\Pr(M_{ij}{=}0\mid S_{ij}^{img}, S_{ij}^{geo})} \nonumber \\
&= O^\pi_{ij}\cdot \mathrm{LR}^{\text{img}}_{ij}\cdot \mathrm{LR}^{\text{geo}}_{ij}, 
\label{eq:odds-ci}
\end{align}
\vspace{-1pt}
where $O^\pi_{ij}=o(\pi_{ij})$ and $\mathrm{LR}^{m}_{ij}=\frac{p(S^{m}_{ij}\mid M_{ij}{=}1)}{p(S^{m}_{ij}\mid M_{ij}{=}0)}$.
Applying Bayes rule to each modality gives:
\begin{align}
O^{\text{img}}_{ij}=o(p^{\text{img}}_{ij})=O^\pi_{ij}\,\mathrm{LR}^{\text{img}}_{ij}, \\
    O^{\text{geo}}_{ij}=o(p^{\text{geo}}_{ij})=O^\pi_{ij}\,\mathrm{LR}^{\text{geo}}_{ij}. 
\end{align}
Hence:
\vspace{-15pt}
\begin{align}
\mathrm{LR}^{\text{img}}_{ij}=\frac{o(p^{\text{img}}_{ij})}{o(\pi_{ij})} \label{eq:LR_img}, \\
\mathrm{LR}^{\text{geo}}_{ij}=\frac{o(p^{\text{geo}}_{ij})}{o(\pi_{ij})}. \label{eq:LR_geo}
\end{align}
Substitute \cref{eq:LR_img} and \cref{eq:LR_geo} into \cref{eq:odds-ci} to obtain
\vspace{-5pt}
\begin{equation}
O^{\text{fuse}}_{ij}
\;=\; O^\pi_{ij}\cdot \frac{o(p^{\text{img}}_{ij})}{o(\pi_{ij})}\cdot \frac{o(p^{\text{geo}}_{ij})}{o(\pi_{ij})}
\;=\; \frac{o(p^{\text{img}}_{ij})\,o(p^{\text{geo}}_{ij})}{o(\pi_{ij})}.
\label{eq:ofuse-odds}
\end{equation}

Finally, writing odds in terms of probabilities and simplifying gives the closed form
\[
p^{\text{fuse}}_{ij}
= \frac{p^{\text{img}}_{ij}\,p^{\text{geo}}_{ij}\,(1-\pi_{ij})}
{p^{\text{img}}_{ij}\,p^{\text{geo}}_{ij}\,(1-\pi_{ij})+\big(1-p^{\text{img}}_{ij}\big)\big(1-p^{\text{geo}}_{ij}\big)\,\pi_{ij}} \ .
\]

\end{proof}
\vspace{-15pt}
\subsection{Noisy-OR Derivation (Eq. 7)}\label{sec:nosiy_or_prof}
\begin{proposition}[Disjunctive Posterior Fusion \textit{(Noisy-OR)}]\label{prop:P_noisy_or}
Let $A_{ij}^{m}$ be a modality activation random variable indicating whether modality $m$ supports the correspondence $(i,j)$. 
We model $A_{ij}^{m}$ as a Bernoulli random variable depending only on its own similarity signal,
\begin{equation}\label{eq:A_def}
    A_{ij}^{m} \mid S_{ij}^{m} \sim \mathrm{Bernoulli}\!\big(p_{ij}^{m}\big),
    \quad
    \Pr(A_{ij}^{m}=1 \mid S_{ij}^{m}) = p_{ij}^{m}.
\end{equation}
Then, under the conditional independence of activations given their respective modality signals, 
the Disjunctive Posterior Fusion (\emph{Noisy-OR}) is given by:
\begin{equation}\label{eq:p_noisy_or_final_supp}
    p^{\text{Noisy-OR}}_{ij}
    = 1 - (1 - p^{\text{img}}_{ij})(1 - p^{\text{geo}}_{ij}).
\end{equation}
\end{proposition}

\begin{proof}[\cref{prop:P_noisy_or}]
By construction, each activation depends only on its own signal and is therefore independent of the other modality’s evidence once that signal is known,
\begin{equation}\label{eq:A_S_indep}
    A_{ij}^{\text{img}} \perp\!\!\!\perp S_{ij}^{\text{geo}} \mid S_{ij}^{\text{img}},
    \qquad
    A_{ij}^{\text{geo}} \perp\!\!\!\perp S_{ij}^{\text{img}} \mid S_{ij}^{\text{geo}}.
\end{equation}
As a result, the two activations are independent given both similarity signals,
\begin{equation}\label{eq:A_A_indep}
    A_{ij}^{\text{img}} \perp\!\!\!\perp A_{ij}^{\text{geo}} 
    \mid (S_{ij}^{\text{img}}, S_{ij}^{\text{geo}}).
\end{equation}
The Disjunctive Posterior Fusion is defined as the probability that at least one modality supports the correspondence:
\begin{equation}
    p^{\text{Noisy-OR}}_{ij}
    = \Pr(A^{\text{img}}_{ij}=1 \,\vee\, A^{\text{geo}}_{ij}=1 
    \mid S^{\text{img}}_{ij}, S^{\text{geo}}_{ij}).
\end{equation}
Using the complementary event and applying \cref{eq:A_A_indep,eq:A_S_indep} gives:
\begin{align}
    p^{\text{Noisy-OR}}_{ij}
    &= 1 - \Pr(A^{\text{img}}_{ij}=0, A^{\text{geo}}_{ij}=0 
    \mid S^{\text{img}}_{ij}, S^{\text{geo}}_{ij}) \nonumber \\[2pt]
    &= 1 - \Pr(A^{\text{img}}_{ij}=0 \mid S^{\text{img}}_{ij}) 
    \Pr(A^{\text{geo}}_{ij}=0 \mid S^{\text{geo}}_{ij}). \label{eq:p_noisy_or_prof1}
\end{align}
Substituting the Bernoulli definition from \eqref{eq:A_def}, 
$\Pr(A^{m}_{ij}=0 \mid S^{m}_{ij}) = 1 - p^{m}_{ij}$, 
into \eqref{eq:p_noisy_or_prof1} yields 
\eqref{eq:p_noisy_or_final_supp}.
\end{proof}

\section{Methodological Details}
\subsection{Consistent Multi-View Generation: Horizontal vs. Temporal concatenation}\label{sec:supp_temporal_concat}

\begin{figure}
    \centering
    \includegraphics[width=0.8\linewidth,  trim={0mm 140mm 0mm 0mm}, clip]{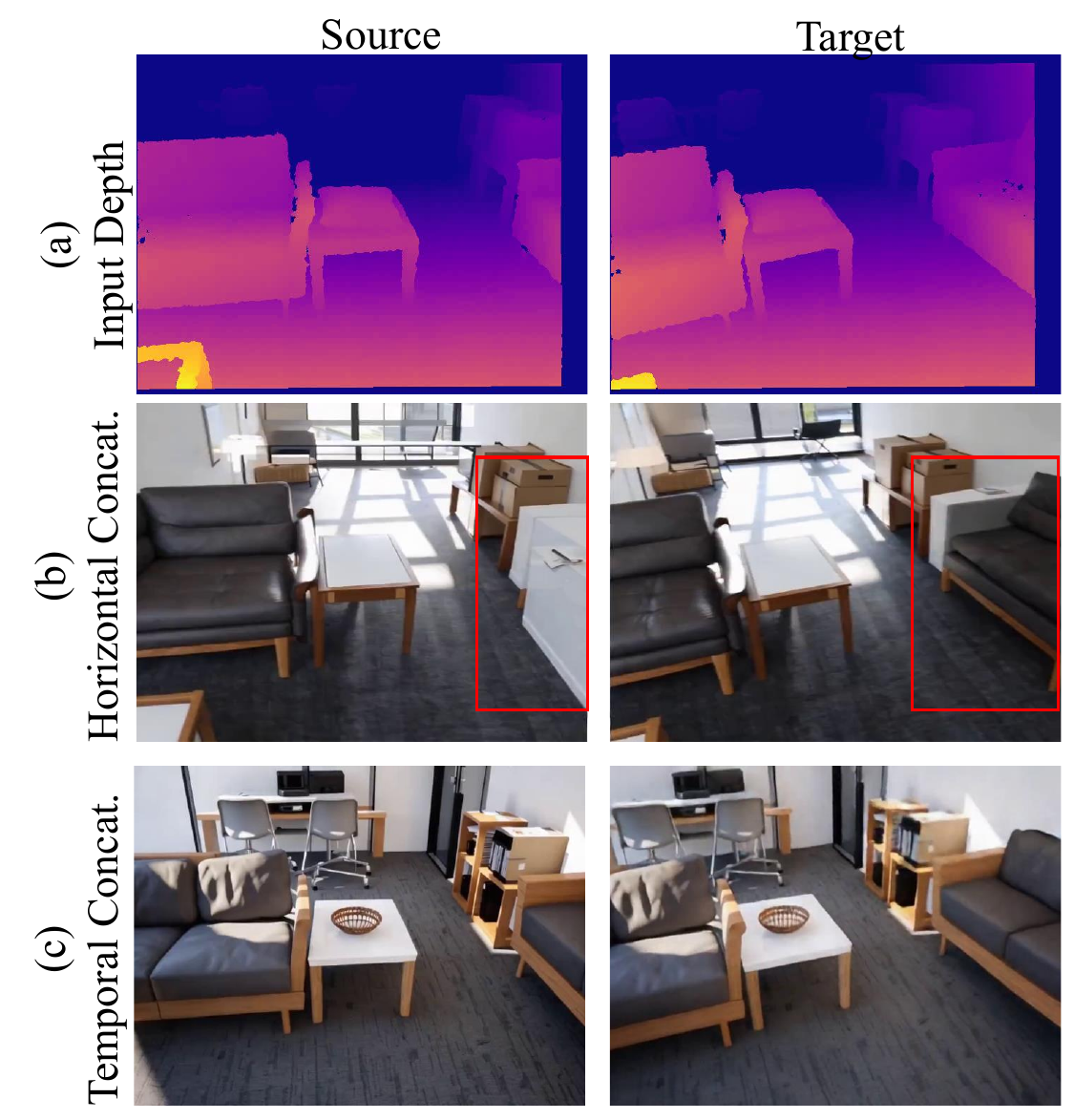}
    \caption{\textbf{WFM Input Formatting.} (a) Input depth maps of the source and target views. (b) Feeding the pretrained WFM with\textit{ horizontally concatenated} depth inputs causes cross-view inconsistencies, e.g., the sofa is mistakenly replaced in the generated source image. (c) Using \textit{temporal concatenation} produces RGB outputs that are geometrically coherent and appearance-consistent between the two views.}
    \vspace{-12pt}
    \label{fig:temp_vs_horz}
\end{figure}
Our method relies on Cosmos-Transfer, a World Foundation Model (WFM) capable of generating multi-view consistent RGB videos from depth inputs. Since Cosmos-Transfer is trained to operate on depth videos, it expects a temporally ordered sequence of depth maps as input.

To adapt this interface to the point cloud registration setting, we construct the WFM input by concatenating the source and target depth sequences along the temporal axis:
\begin{equation}
D^{\text{in}} = \big[\, D^{\text{src}}_1, \dots, D^{\text{src}}_L,\; D^{\text{tgt}}_1, \dots, D^{\text{tgt}}_L \,\big].
\end{equation}

This temporal concatenation forms a single depth video containing the two fragments sequentially. While this configuration does not perfectly match the model’s original training distribution, it remains physically plausible and statistically closer to natural camera motion than alternative layouts such as spatial horizontal concatenation.

The key advantage of temporal concatenation is that it preserves the pretrained multi-view consistency priors of the WFM. Because the model was trained on temporally coherent video sequences, it naturally propagates geometric and semantic information across adjacent frames, including across the transition between the source and target segments.

In contrast, horizontally concatenating the two fragments would introduce an artificial spatial discontinuity that is not present in the model’s training distribution, often resulting in weaker geometric coherence between the generated views. A qualitative comparison between the two strategies is illustrated in \cref{fig:temp_vs_horz}.

Due to the autoregressive nature of video generation, minor visual transients may appear near the transition between the source and target segments. To mitigate this effect, we discard several ``safeguard'' frames around the midpoint of the generated video before extracting features for correspondence estimation.

\subsection{View Selection Strategy}
\label{sec:view_selection}
\begin{figure}
    \centering
    \includegraphics[width=0.9\linewidth]{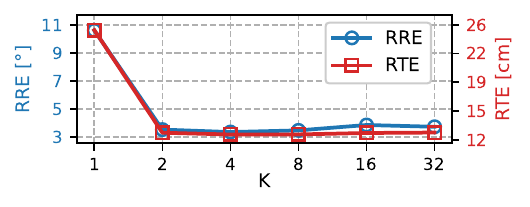}
    \vspace{-10pt}
    \caption{\textbf{Effect of View Selection ($K$).} Registration performance measured by \textcolor{rreclr}{Relative Rotation Error (RRE)} and \textcolor{rteclr}{Relative Translation Error (RTE)} as a function of the number of selected views $K$. Performance saturates for $K\ge 4$, indicating that only a few representative views are sufficient for stable registration.}
    \vspace{-15pt}
    \label{fig:view_selection}
\end{figure}

MASt3R extracts features using a cross-attention decoder that jointly processes pairs of images, meaning that the representation of a given image depends on the image with which it is paired. Consequently, evaluating a source image against different target views produces distinct conditioned feature maps.
To exploit this property, we sample $K$ views uniformly from the $L$ frames of the generated source and target sequences and evaluate all pairwise combinations. This produces $K^2$ conditioned feature maps per domain, denoted as $F^{\text{img}}_{\text{src}}, F^{\text{img}}_{\text{tgt}} \in \mathbb{R}^{K^2 \times H \times W \times d_{\text{img}}}$.

While increasing $K$ improves viewpoint coverage, it also leads to quadratic growth in the number of evaluated image pairs, motivating a careful view-selection strategy. As shown in \cref{fig:view_selection}, both RRE and RTE quickly saturate as $K$ increases, reflecting the high correlation between frames in the generated sequences. This suggests that selecting $K \ll L$ is sufficient to maintain view diversity while keeping the computational cost manageable.

\subsection{C-GenReg for LiDAR Data}\label{sec:lidar_supp}
\begin{figure*}[t]
    \centering
    \includegraphics[width=1\linewidth]{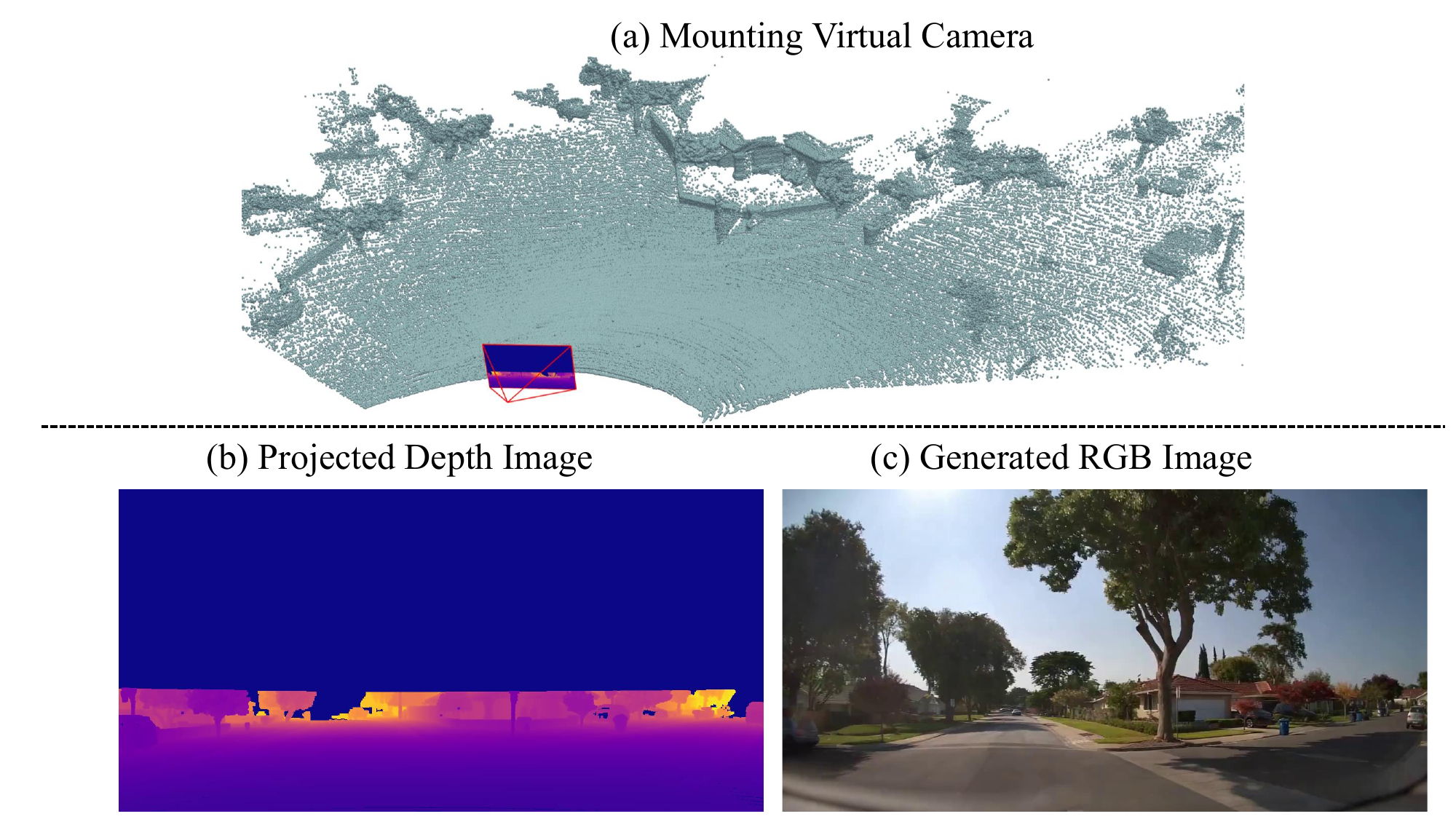}
    \caption{\textbf{C-GenReg LiDAR Input Pipeline: }(a) A virtual camera is configured into the LiDAR scan. (b) The LiDAR points are projected into a depth image. (c) The resulting depth map is fed into the generative model to produce an aligned RGB image.}
    \label{fig:lidar_virtual_cam}
\end{figure*}
Since Cosmos-Transfer expects depth images as input, we introduce a preprocessing stage that converts raw 3D LiDAR scans into a depth-image representation. Following \cite{cosmos_dream_drive}, we project each LiDAR point cloud onto a \textit{virtual camera}. Choosing an appropriate virtual camera model is essential, as LiDAR sensors cover an extremely wide field of view (FOV).

In line with Cosmos-Transfer, we use an $f$-$\theta$ virtual camera instead of a standard pinhole model. Wide-angle perception systems, typical in autonomous driving, require ray–angle–based projection models that accurately handle $180^\circ+$ FOV and avoid the extreme nonlinear distortions produced by perspective projection at large incidence angles \cite{scaramuzza2006flexible, courbon2007generic}. The $f$-$\theta$ formulation also aligns well with the optics of wide-FOV imaging systems commonly used in robotics and AV platforms.

\Cref{fig:lidar_virtual_cam} shows the full conversion pipeline: attaching the virtual camera, rendering a depth map from the LiDAR scan, and passing the resulting depth image to the World Foundation Model to generate the corresponding RGB frame. For $360^\circ$ LiDAR, our approach can be naturally extended by deploying multiple virtual cameras with overlapping FOVs, leveraging the WFM’s multi-view consistency to produce a coherent full-view RGB generation.

\subsection{Fusion Method - Point Matching Performance}\label{sec:fuse_ablation_supp}
Accurate registration relies heavily on producing reliable point-to-point correspondences. To evaluate the effect of our probabilistic fusion operators on this stage, we compare Noisy-AND and Noisy-OR on the point-matching task. \cref{fig:and_vs_or} shows the resulting precision–recall curves evluated on the 3DMatch dataset. A correspondence is counted as correct if, under the ground-truth transformation, the matched points lie within $5cm$ of each other. Across the entire recall range, Noisy-AND consistently attains higher precision than Noisy-OR. This behavior is expected: Noisy-AND emphasizes matches that are simultaneously confident in both modalities, whereas Noisy-OR tends to admit a larger set of correspondences, including lower-quality ones, which reduces precision.

\section{Additional Implementation Details}\label{sec:implementation_details_supp}
For RGB generation, we employ \textit{Cosmos-Transfer1-7B (Depth)} for indoor datasets and Cosmos-\textit{Transfer1-7B-Sample-AV (LiDAR)} for outdoor datasets. Input depth maps are resized to $960\!\times\!704$ for indoor data and $1280\!\times\! 640$ for outdoor data to match the expected Cosmos input resolutions. Cosmos is run with the following parameters: CFG=7, $\sigma_{\max}\!\!=\!80$, a spatiotemporal control weight of 1, 35 denoising steps, and a target frame rate of 30\,fps. The output resolution matches the input depth resolution, and all inference is performed on an NVIDIA RTX A6000 GPU.

For the VFM pathway, we use MASt3R with an Encoder ViT-L and Decoder ViT-B configuration. Input RGB images are resized to $512\!\times\!384$ prior to feature extraction, and we use only the descriptor head ($d_{\text{img}}=24$).

For the geometric branch, indoor point clouds are voxelized at $2.5$\,cm and outdoor point clouds at $5$\,cm. We extract geometric features using GeoTransformer, employing the official 3DMatch checkpoint for indoor scenes and the KITTI checkpoint for outdoor scenes, producing the geometric descriptors ($d_{\text{geo}}=256$).

All components, Cosmos, MASt3R, and GeoTransformer, are used with their publicly released pretrained weights and remain completely frozen in our pipeline.

\section{Runtime Analysis}\label{sec:runtime_supp}
\begin{table}[t]
\centering
\footnotesize
\resizebox{\columnwidth}{!}{
\begin{tabular}{lccccc}
\toprule
\textbf{Method} & \textbf{WFM (s)} & \textbf{VFM (s)} & \textbf{Geom. (s)} & \textbf{Pose (s)} & \textbf{Total (s)} \\
\midrule
GeoTransformer & - & - & 0.075 & 1.558 & 1.633 \\
Ours           & 507.0 & 0.973 & 0.075 & 0.066 & 508.1 \\
\bottomrule
\end{tabular}
}
\caption{\textbf{Runtime Analysis.} Runtime per registration problem measured on a single NVIDIA RTX A6000 GPU.}
\label{tab:runtime}
\end{table}

We report the runtime breakdown of C-GenReg and compare it with GeoTransformer in \cref{tab:runtime}. Since GPCR code is not publicly available, a direct runtime comparison with that method cannot be provided.
The runtime of C-GenReg is dominated by the World Foundation Model (WFM), which generates multi-view RGB videos from the input depth sequences. As shown in \cref{tab:runtime}, this stage accounts for almost the entire runtime ($507$s), while the remaining components are lightweight: the Vision Foundation Model (VFM) used for correspondence extraction requires less than one second, and the geometric matching and pose estimation stages take only a fraction of a second.

This cost reflects the use of a powerful pretrained generative prior in a fully training-free pipeline. Importantly, recent work on distilling Cosmos Transfer models \cite{cosmos_cookbook_distilling_transfer1_2025} reports up to a $72\times$ inference speedup, which would reduce the runtime of our pipeline to approximately $\sim7$s. Additional speedups may also be achieved by lowering the video generation rate.

\begin{table}[t]
\centering
\setlength{\tabcolsep}{1.8pt} % compact horizontal spacing
\renewcommand{\arraystretch}{0.98} % compact vertical spacing
\resizebox{\columnwidth}{!}{%
\begin{tabular}{@{}l
                c c c c c
                c c c c c@{}}
\Xhline{1.3pt}
\multicolumn{1}{c}{} &
\multicolumn{5}{c}{\textbf{Rotation} [deg]} &
\multicolumn{5}{c}{\textbf{Translation} [cm]} \\
\cmidrule(lr){2-6} \cmidrule(l){7-11}
\textbf{Method} &
\multicolumn{3}{c}{\textbf{Accuracy} $\uparrow$} &
\multicolumn{2}{c}{\textbf{Error} $\downarrow$} &
\multicolumn{3}{c}{\textbf{Accuracy} $\uparrow$} &
\multicolumn{2}{c}{\textbf{Error} $\downarrow$} \\
\cmidrule(lr){2-4} \cmidrule(lr){5-6} \cmidrule(lr){7-9} \cmidrule(l){10-11}
& $5$ & $10$ & $45$ & Mean & Med. &
$5$ & $10$ & $25$ & Mean & Med. \\
\Xhline{1.3pt}
FCGF
  & 70.2 & 87.7 & 96.2 & 9.5 & 3.3
  & 27.5 & 58.3 & 82.9 & 23.6 & 8.3 \\
GeoTransformer
  & 94.0 & 96.8 & 98.1 & 4.3 & 1.0
  & 79.2 & 92.0 & 96.7 & 8.2 & 2.5 \\
\textbf{C-GenReg (Ours)}
  & \textbf{99.4} & \textbf{99.8} & \textbf{99.9} & \textbf{1.1} & \textbf{0.9}
  & \textbf{87.5} & \textbf{97.1} & \textbf{99.3} & \textbf{3.0} & \textbf{1.9} \\
\Xhline{1.3pt}
\end{tabular}%
}
\caption{\textbf{ScanNet Original Benchmark.} Rotation and translation accuracy (\% of pairs within RRE/RTE thresholds in $deg$ and $cm$ respectively) and mean/median error on the ScanNet original Split benchmark. Best results are in \textbf{bold}.}
\label{tab:scannet_org}
\vspace{-15pt}
\end{table}

\section{Additional Results}\label{sec:addtional_res_supp}
\begin{figure}
    \centering
    \includegraphics[width=1\linewidth]{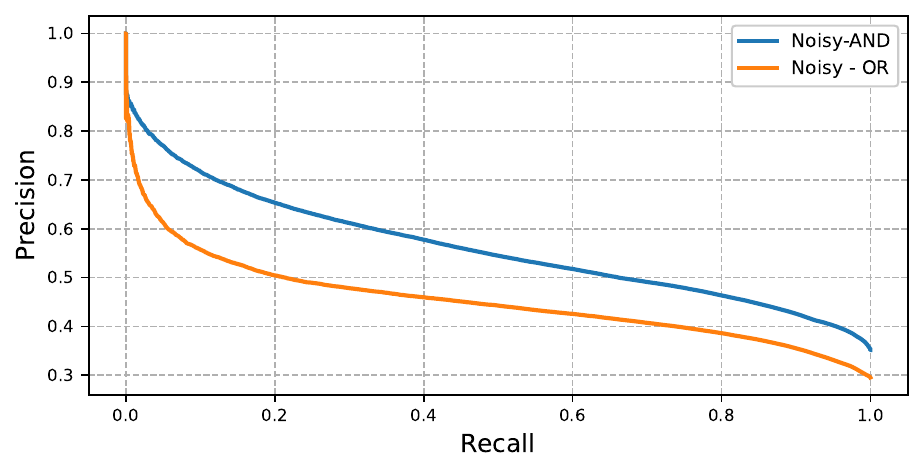}
    \caption{\textbf{Matching Performance Comparison of Noisy-AND vs. Noisy-OR.} Precision–recall curves comparing the two probabilistic fusion operators on the point-matching task (a match is correct if within $5 cm$ under the ground-truth transformation). Noisy-AND consistently achieves higher precision at similar recall rates.}
    \label{fig:and_vs_or}
\end{figure}

\subsection{ScanNet Original Benchmark}\label{sec:scannet_org_supp}
In \cref{tab:scannet_org}, we report registration performance on the original ScanNet benchmark, which consists of relatively easy pairs sampled 20 frames apart, resulting in modest motion between the source and target scans. C-GenReg achieves clear improvements across all metrics, with particularly strong gains in translation accuracy compared to FCGF and GeoTransformer, thus demonstrating that augmenting geometric features with our RGB-generated branch provides a consistent performance boost and serves as an effective enhancement to existing registration pipelines.

\subsection{Low-Overlap Benchmarks}
\label{sec:low_overlap}
\begin{table}[t]
\centering
\footnotesize
\setlength{\tabcolsep}{3.5pt}
\begin{tabular}{lcccc}
\toprule
 & \multicolumn{2}{c}{\textbf{Lo3DMatch}} & \multicolumn{2}{c}{\textbf{LoWaymo}} \\
\cmidrule(lr){2-3} \cmidrule(lr){4-5}
\textbf{Method} & \textbf{RRE} & \textbf{RTE} & \textbf{RRE} & \textbf{RTE} \\
\midrule
GeoTransformer & 21.10 & 53.46 & 19.72 & 9.04 \\
Ours           & \textbf{14.57} & \textbf{45.49} & \textbf{4.95} & \textbf{1.66} \\
\bottomrule
\end{tabular}
\caption{\textbf{Low-Overlap Results.} Mean RRE (degrees) and mean RTE (cm for Lo3DMatch and m for LoWaymo).}
\label{tab:low_overlap}
\vspace{-8pt}
\end{table}
To further evaluate the robustness of C-GenReg in challenging scenarios, we conduct experiments on low-overlap registration benchmarks. Specifically, we evaluate on the Lo3DMatch benchmark and on a low-overlap split of the Waymo dataset, where the overlap between point clouds is limited to be less than $30\%$ (\cref{tab:low_overlap}). 
As expected, performance degrades compared to high-overlap cases due to the reduced geometric overlap between scans. Nevertheless, C-GenReg consistently outperforms the geometry-only baseline GeoTransformer across both datasets. On Lo3DMatch, C-GenReg reduces the rotation error from $21.10^\circ$ to $14.57^\circ$ and the translation error from $53.46$\,cm to $45.49$\,cm. A more substantial improvement is observed on the low-overlap Waymo benchmark, where C-GenReg reduces the rotation error from $19.72^\circ$ to $4.95^\circ$ and the translation error from $9.04$\,m to $1.66$\,m.

These results highlight the benefit of incorporating generative priors into the registration pipeline. Even when the geometric overlap between scans is limited, the WFM-based image generation remains consistent in the shared regions of the scene. Combined with the probabilistic match-then-fuse formulation, this enables C-GenReg to recover reliable correspondences despite the sparsity of overlapping geometry, leading to improved registration accuracy.

\subsection{Qualitative Examples}\label{sec:qualitative_supp}
We present additional qualitative results of C-GenReg across both indoor and outdoor benchmarks.
\cref{fig:qual_matches_3dmatch} and \cref{fig:qual_matches_waymo} illustrate registration outcomes on the 3DMatch and Waymo datasets, respectively. Each example shows the generated RGB views with a subset of color-coded correspondence matches, together with the same matches visualized directly on the 3D point clouds.

\cref{fig:gen_rgb_3dmatch}, \cref{fig:gen_rgb_scannet} and \cref{fig:gen_rgb_waymo} highlight the generative capabilities of the employed World Foundation Model across the evaluated datasets, and especially the geometric coherence and multi-view consistency of the views generated by the C-GenReg pipeline, as these are the essential components towards the success of the C-GenReg registration. For each scene, we visualize the input depth maps and the generated RGB outputs, demonstrating strong multi-view consistency between source and target views as well as geometric coherence between the underlying depth structure and the synthesized images.

\begin{figure*}
    \centering
    \includegraphics[
        page=1,
        width=\linewidth,
        trim={0mm 0mm 90mm 0mm},
        clip
    ]{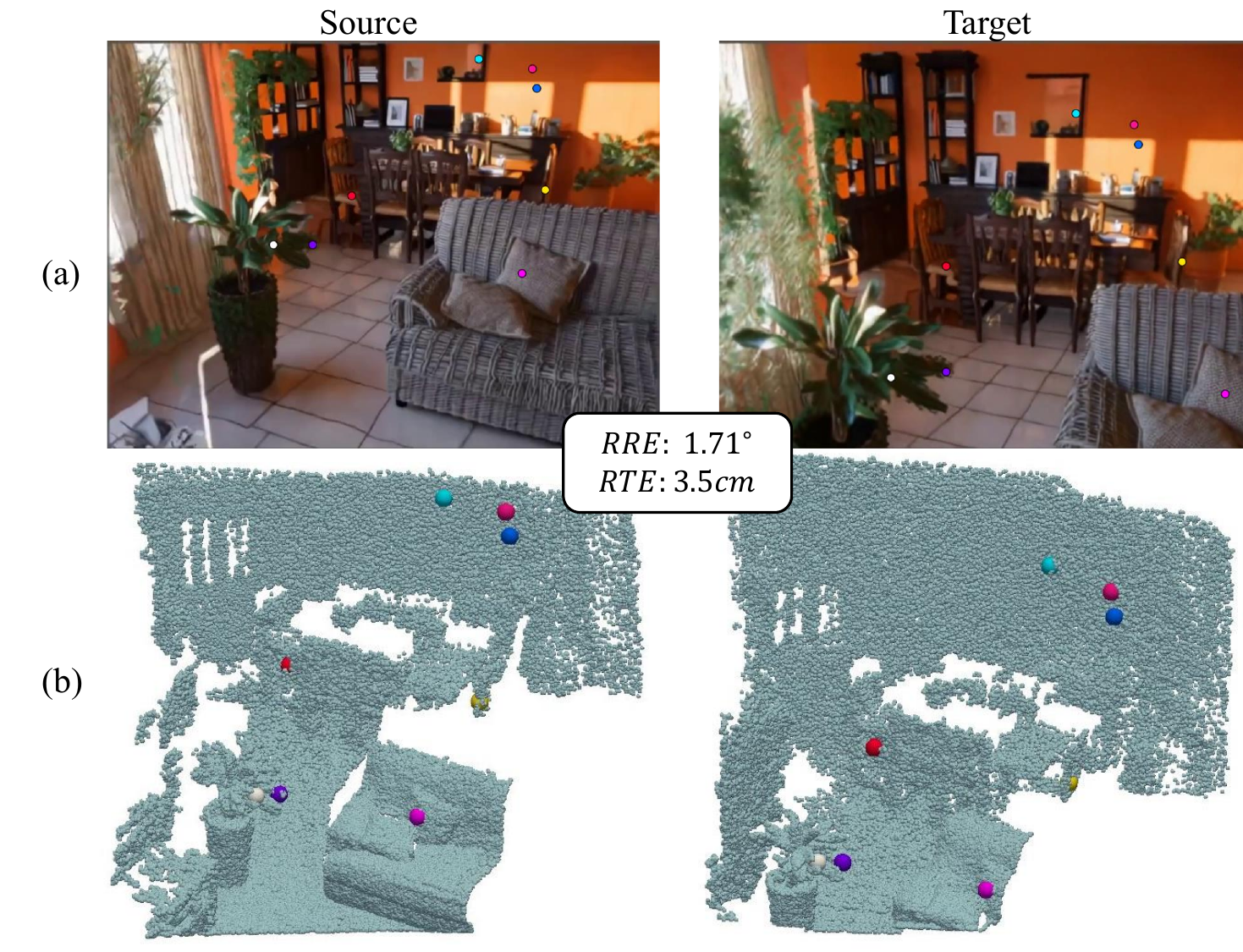}

    \caption{
        \textbf{Qualitative registration example from C-GenReg on the 3DMatch dataset.}
        Generated source and target images with a subset of matched keypoints 
        (same color indicates correspondence), and the same correspondences visualized 
        on the source and target 3D point clouds. The resulting rotation error 
        (RRE, °) and translation error (RTE, cm) are reported as well.
    }
    \label{fig:qual_matches_3dmatch}
    \vspace{-10pt}
\end{figure*}

\begin{figure*}[t]
    \centering

    % -------- Example 1 --------
    \begin{subfigure}{0.95\linewidth}
        \centering
        \includegraphics[
            page=3,
            width=\linewidth,
            trim={0mm 0mm 0mm 0mm},
            clip
        ]{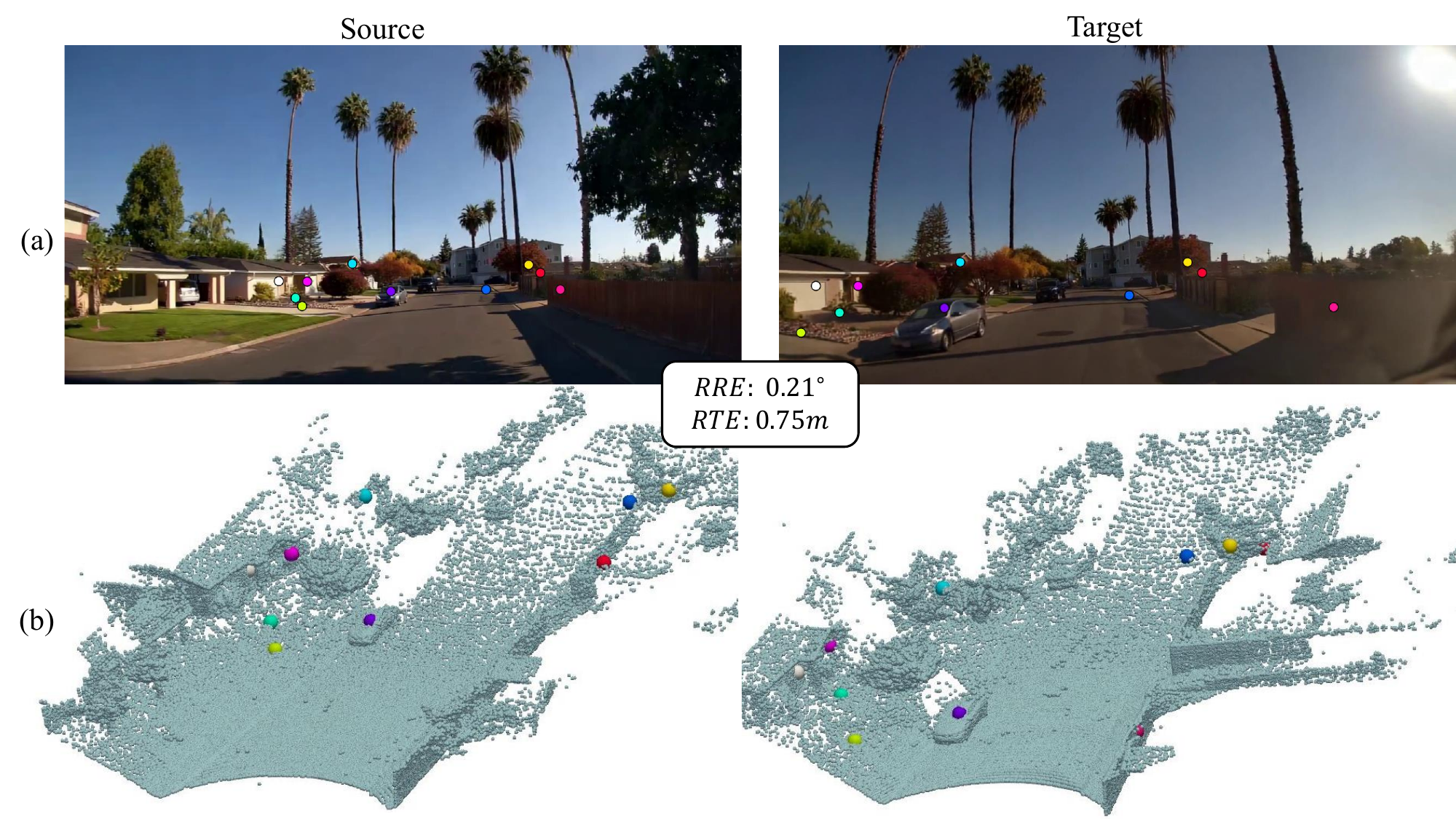}
        \caption*{Example 1}
    \end{subfigure}

    \vspace{6pt}

    % -------- Example 2 --------
    \begin{subfigure}{0.95\linewidth}
        \centering
        \includegraphics[
            page=4,
            width=\linewidth,
            trim={0mm 0mm 0mm 0mm},
            clip
        ]{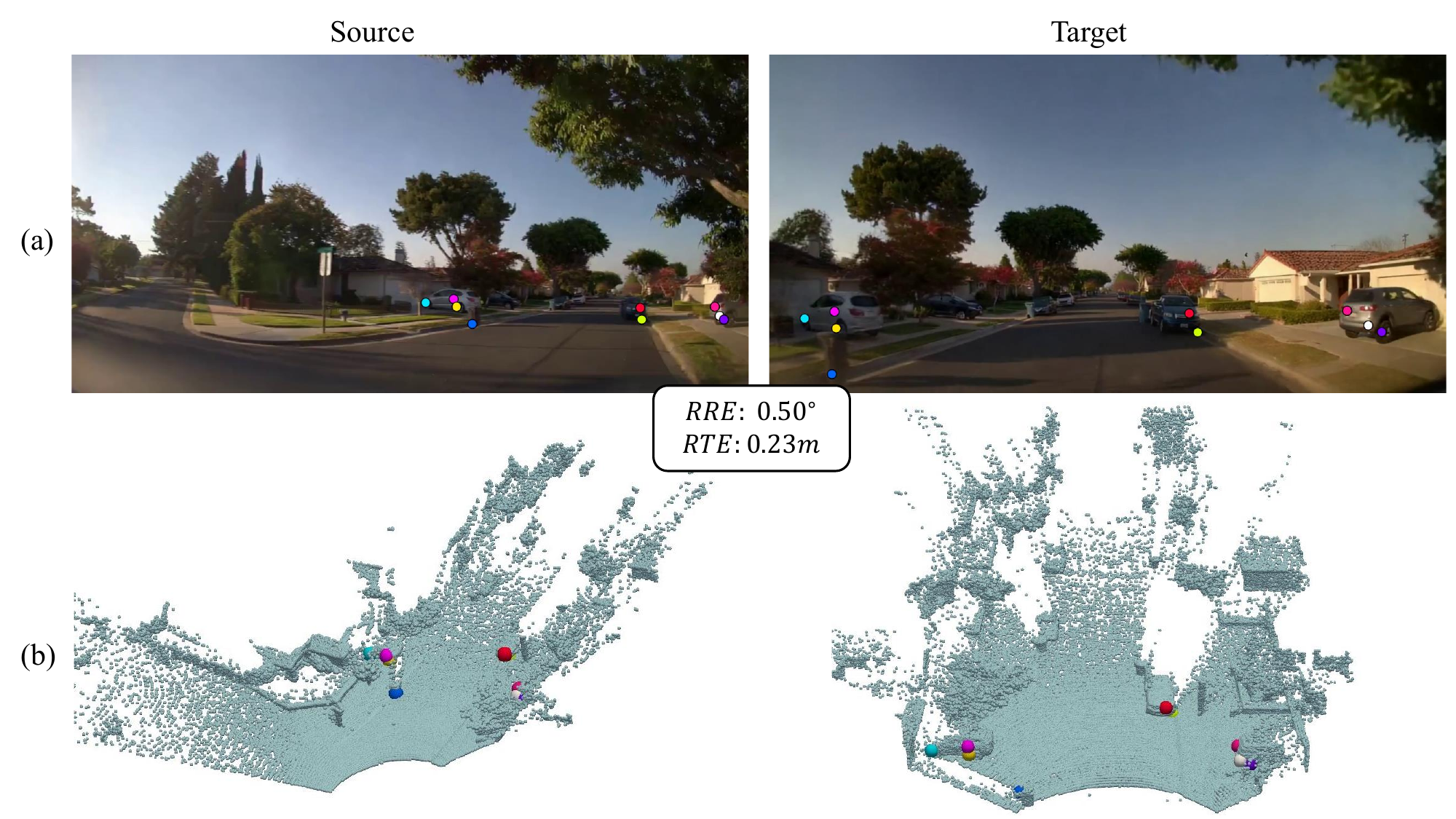}
        \caption*{Example 2}
    \end{subfigure}

    \caption{
        \textbf{Qualitative registration examples of C-GenReg on the Waymo dataset.}
       Row (a) shows generated source and target images with a subset of matched keypoints 
        (same color indicates correspondence). Row (b) shows the same correspondences visualized 
        on the source and target 3D point clouds. The resulting rotation 
        error (RRE, °) and translation error (RTE, m) are also reported.
    }
    \label{fig:qual_matches_waymo}
    \vspace{-12pt}
\end{figure*}

\begin{figure*}[!p]
    \centering

    % =========================================================
    % Figure 11
    % =========================================================
    \setcounter{subfigure}{0}
    \begin{minipage}{\textwidth}
        \centering

        \begin{subfigure}{0.316\textwidth}
            \centering
            \includegraphics[
                page=5,
                width=\linewidth,
                trim={0mm 0mm 90mm 0mm},
                clip
            ]{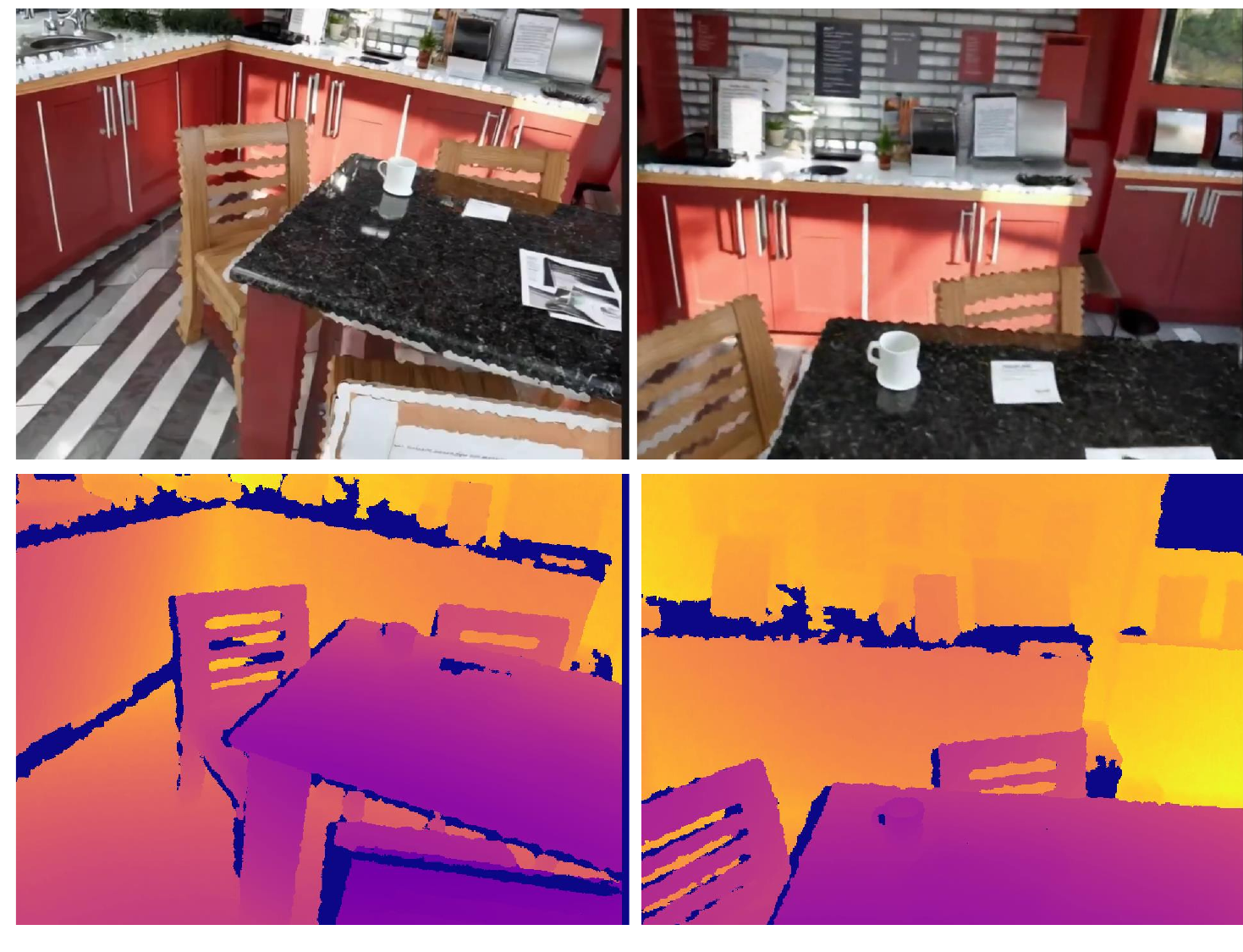}
            \caption{Example 1}
        \end{subfigure}
        \hfill
        \begin{subfigure}{0.316\textwidth}
            \centering
            \includegraphics[
                page=6,
                width=\linewidth,
                trim={0mm 0mm 90mm 0mm},
                clip
            ]{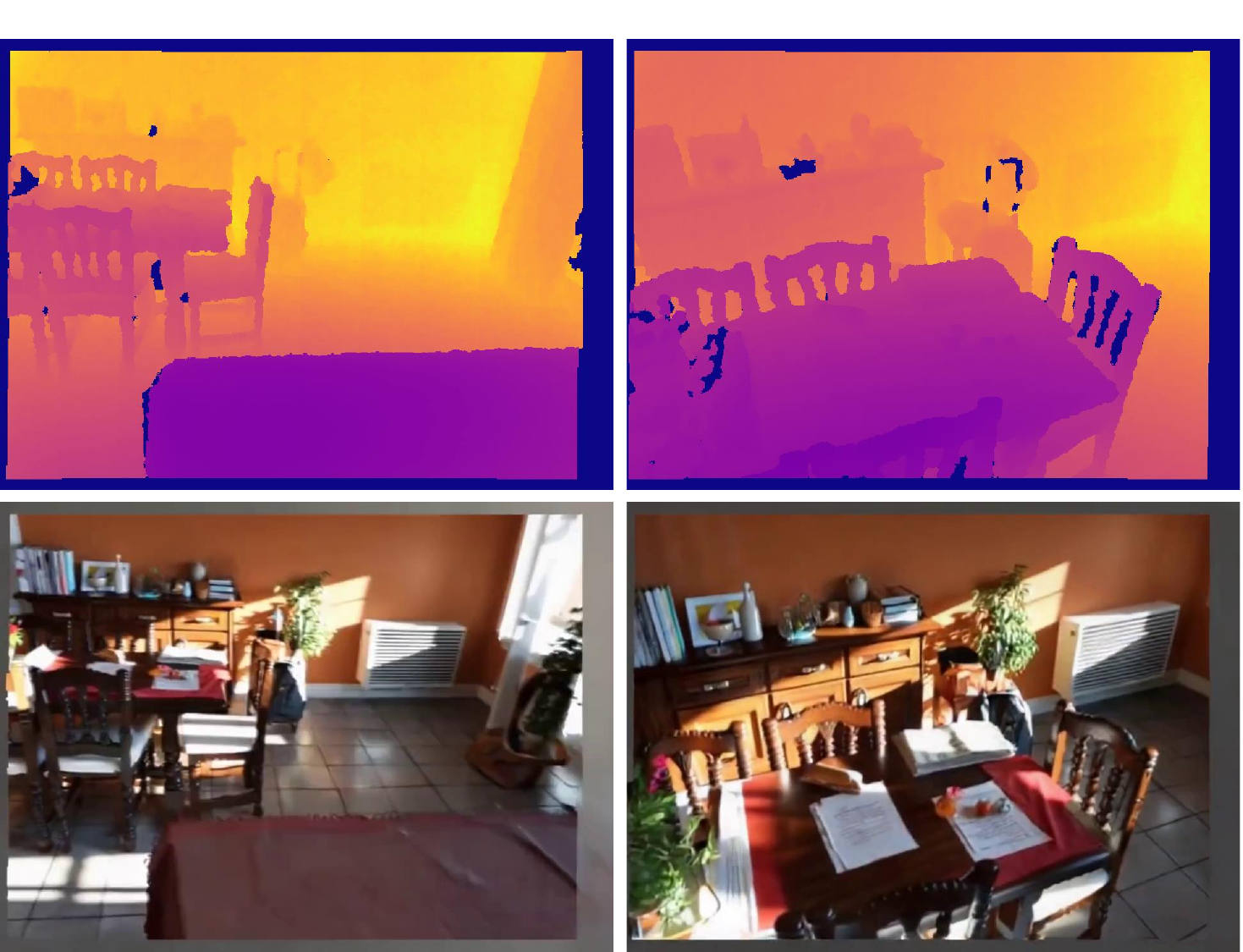}
            \caption{Example 2}
        \end{subfigure}
        \hfill
        \begin{subfigure}{0.316\textwidth}
            \centering
            \includegraphics[
                page=7,
                width=\linewidth,
                trim={0mm 0mm 90mm 0mm},
                clip
            ]{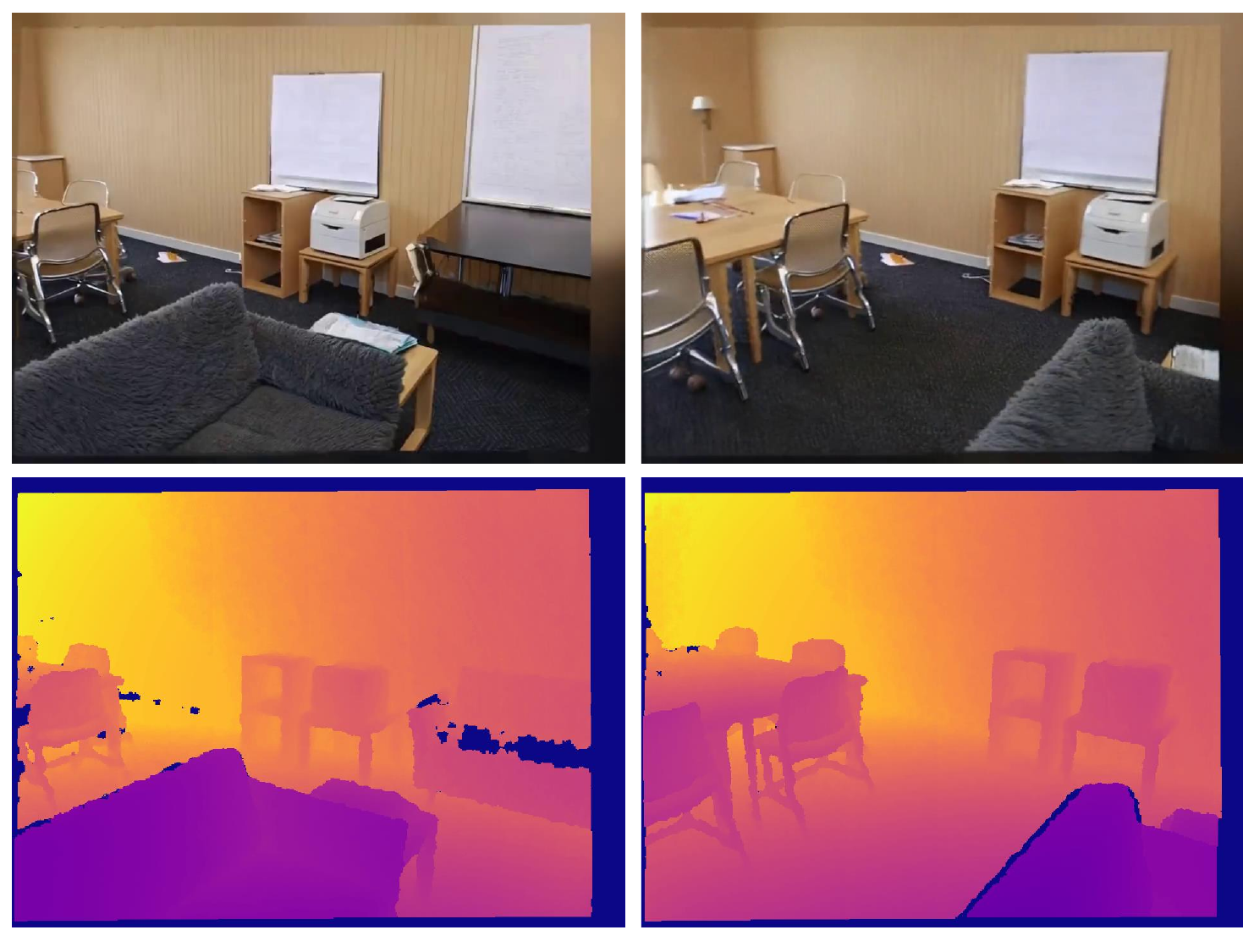}
            \caption{Example 3}
        \end{subfigure}

        \setcounter{figure}{10}
        \captionof{figure}{
            \textbf{Multi-view consistent RGB generation from depth on 3DMatch.}
            Three representative synthetic RGB examples generated from depth.
            The paired views remain geometrically and visually consistent.
        }
        \label{fig:gen_rgb_3dmatch}
    \end{minipage}

    \vspace{4pt}

    % =========================================================
    % Figure 12
    % =========================================================
    \setcounter{subfigure}{0}
    \begin{minipage}{\textwidth}
        \centering

        \begin{subfigure}{0.316\textwidth}
            \centering
            \includegraphics[
                page=12,
                width=\linewidth,
                trim={0mm 0mm 85mm 0mm},
                clip
            ]{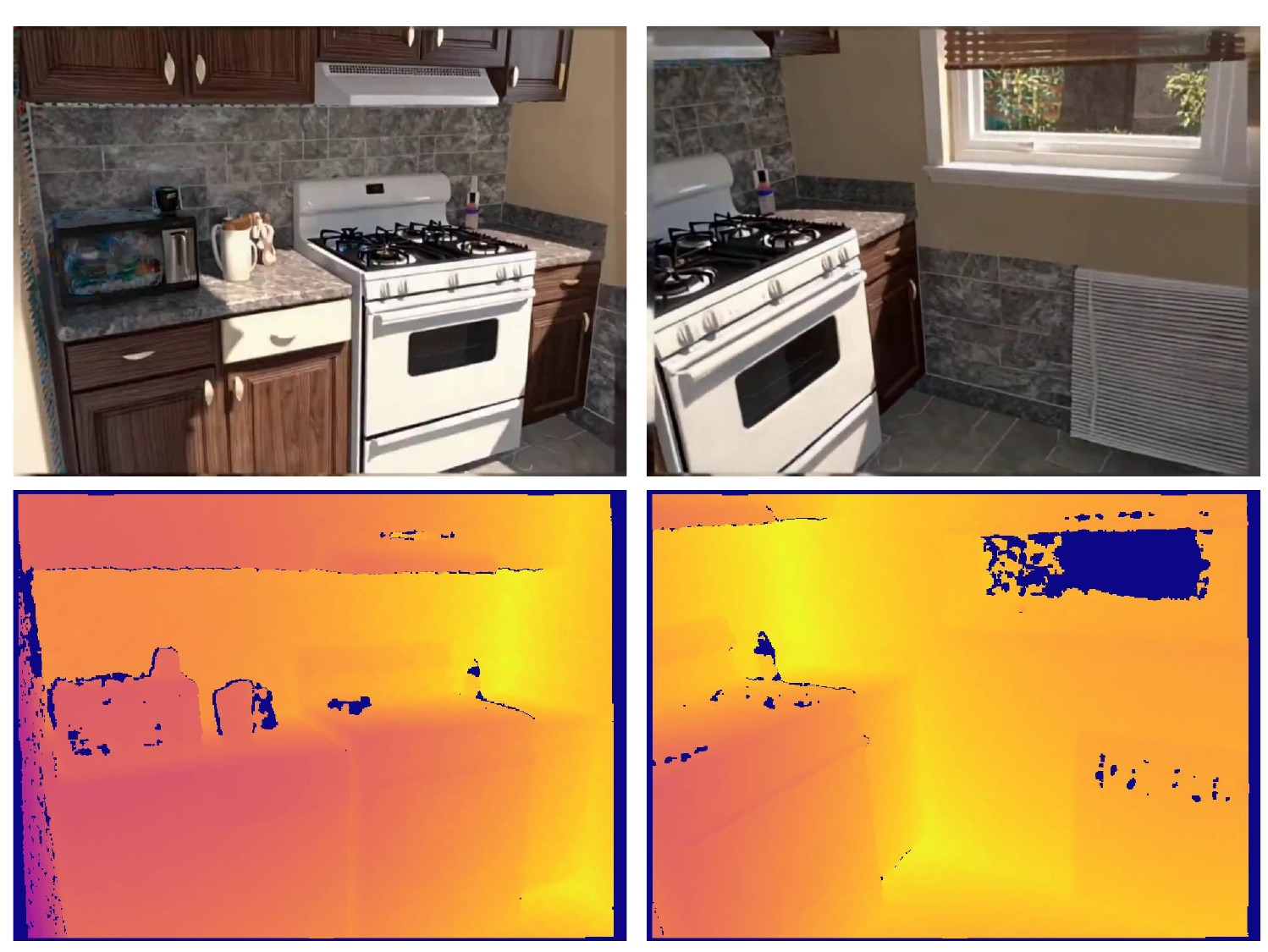}
            \caption{Example 1}
        \end{subfigure}
        \hfill
        \begin{subfigure}{0.316\textwidth}
            \centering
            \includegraphics[
                page=13,
                width=\linewidth,
                trim={0mm 0mm 85mm 0mm},
                clip
            ]{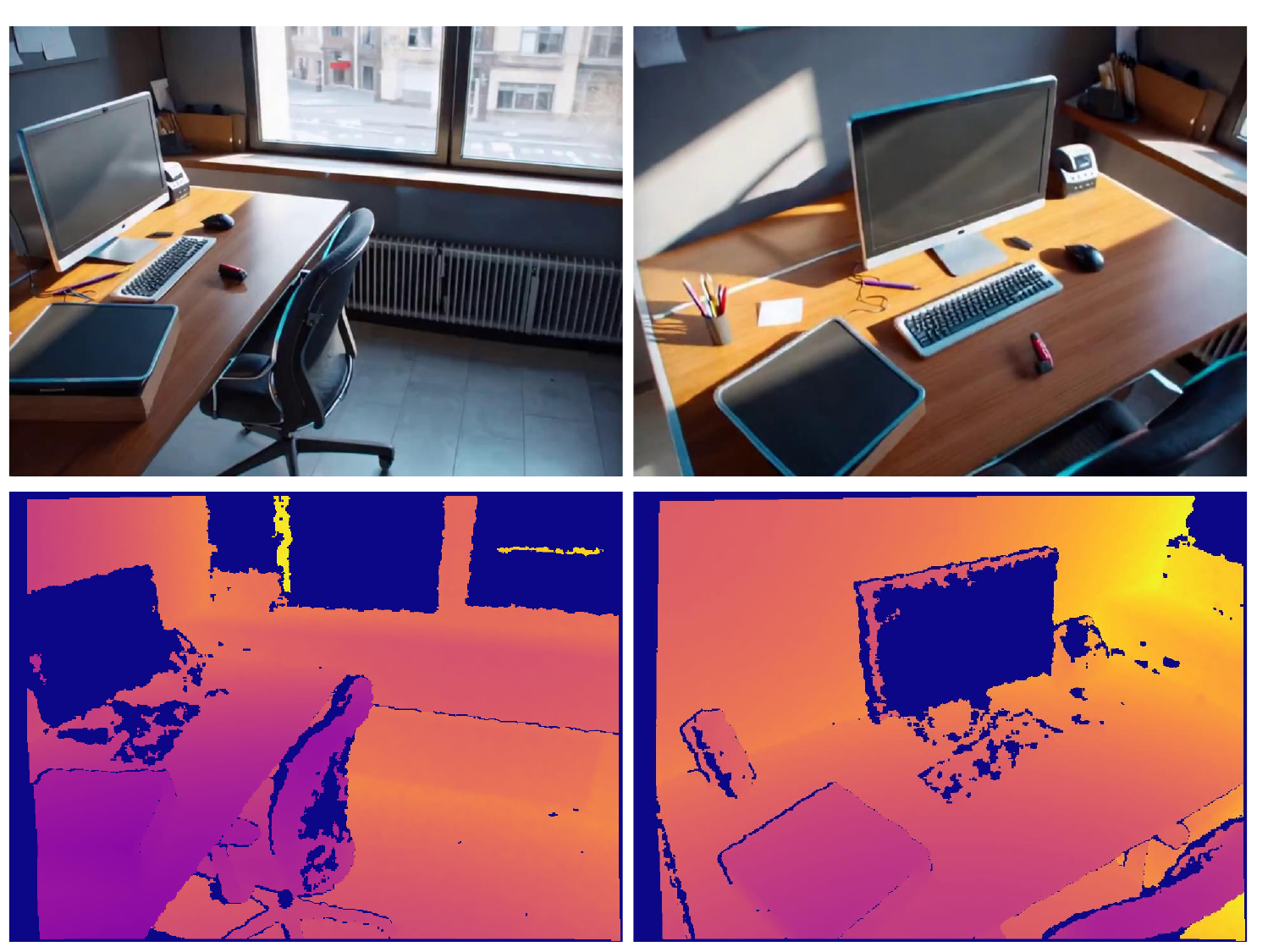}
            \caption{Example 2}
        \end{subfigure}
        \hfill
        \begin{subfigure}{0.316\textwidth}
            \centering
            \includegraphics[
                page=14,
                width=\linewidth,
                trim={0mm 0mm 85mm 0mm},
                clip
            ]{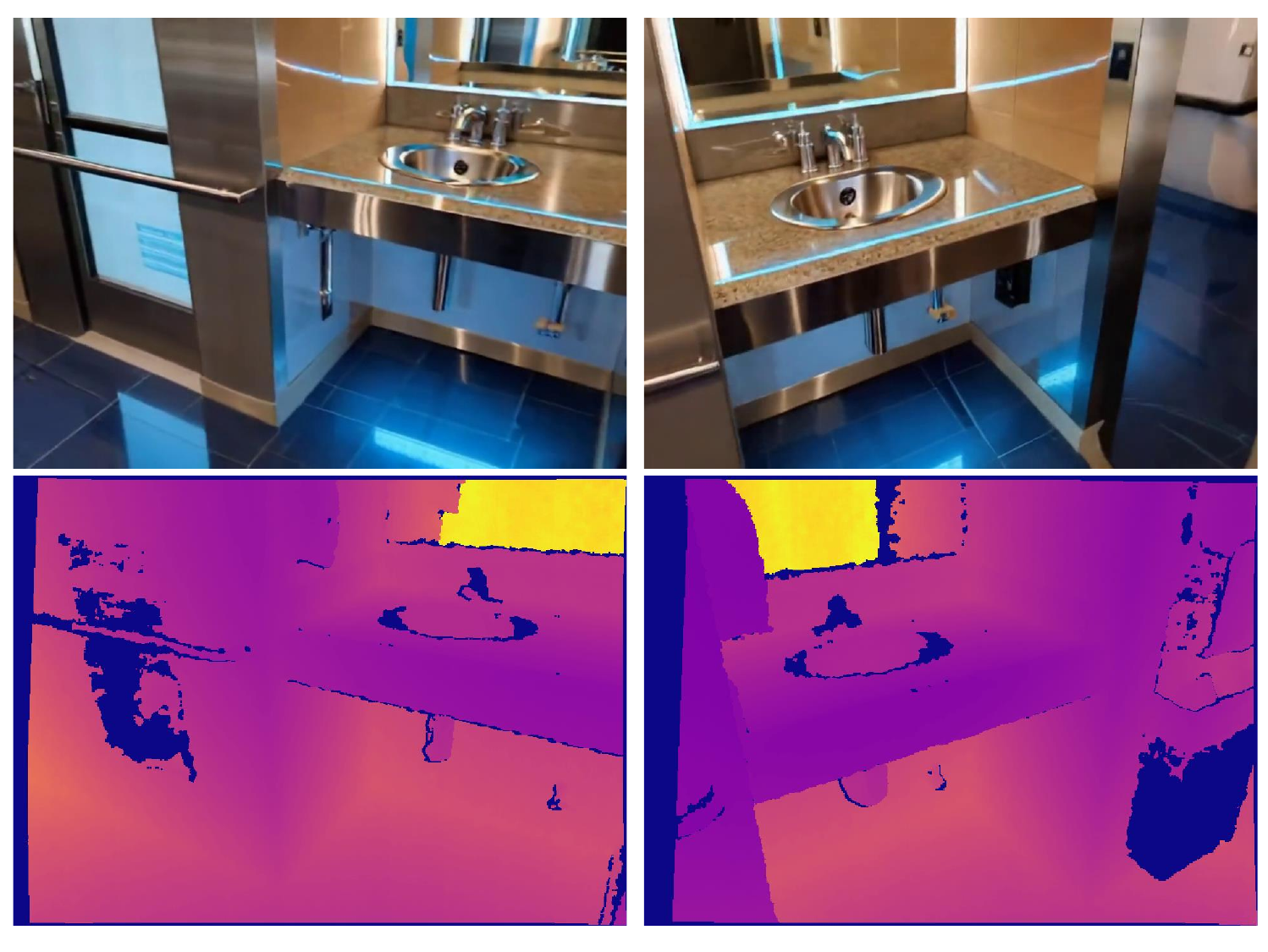}
            \caption{Example 3}
        \end{subfigure}

        \captionof{figure}{
            \textbf{Multi-view consistent RGB generation from depth on ScanNet.}
            Three representative synthetic RGB examples from indoor depth scans.
            The synthesized frames preserve layout and structure across viewpoints.
        }
        \label{fig:gen_rgb_scannet}
    \end{minipage}

    \vspace{4pt}

    % =========================================================
    % Figure 13
    % =========================================================
    \setcounter{subfigure}{0}
    \begin{minipage}{\textwidth}
        \centering

        \begin{subfigure}{0.445\textwidth}
            \centering
            \includegraphics[
                page=8,
                width=\linewidth,
                trim={10mm 5mm 10mm 5mm},
                clip
            ]{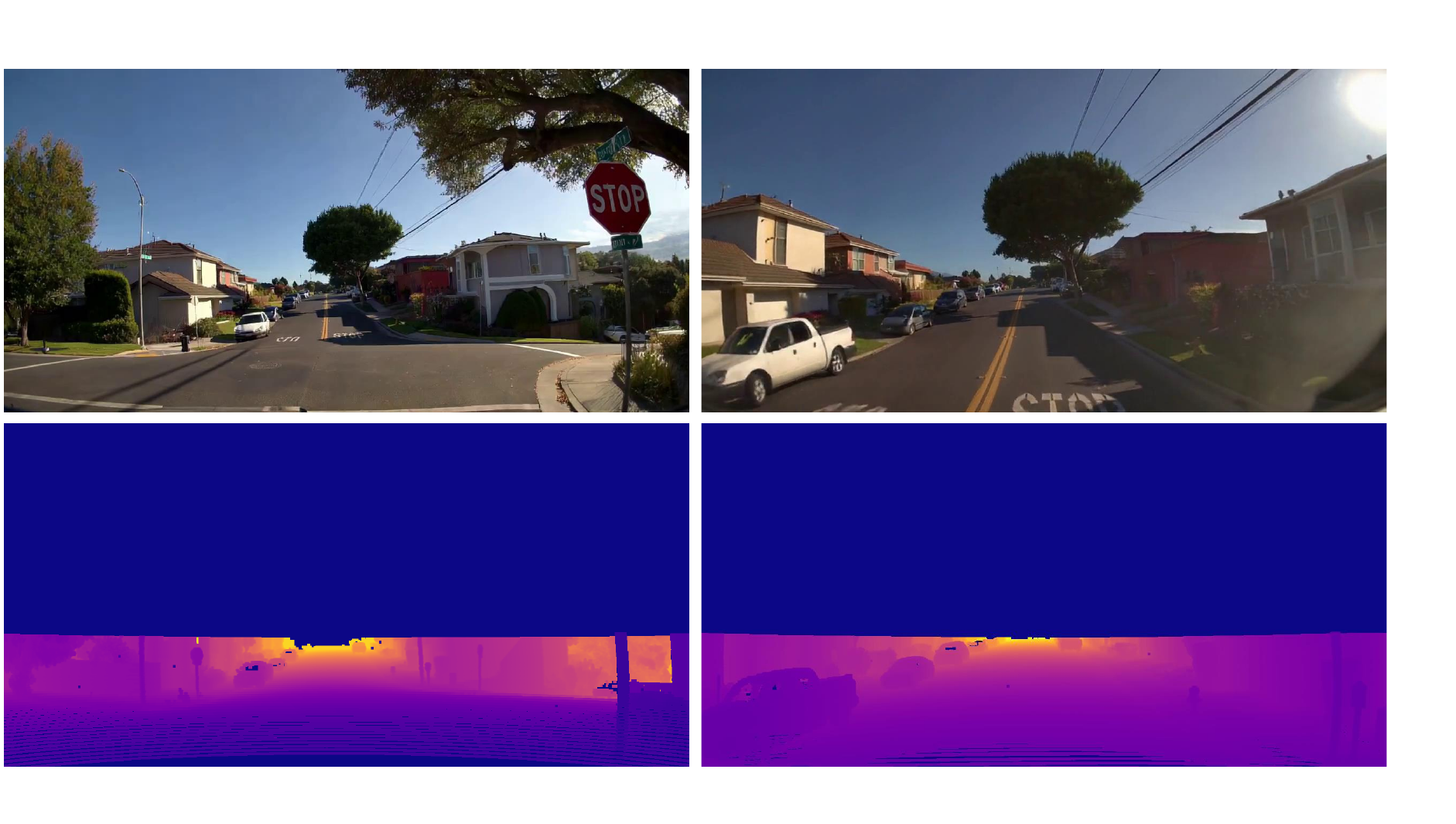}
            \caption{Example 1}
        \end{subfigure}
        \hfill
        \begin{subfigure}{0.445\textwidth}
            \centering
            \includegraphics[
                page=9,
                width=\linewidth,
                trim={0mm 20mm 20mm 0mm},
                clip
            ]{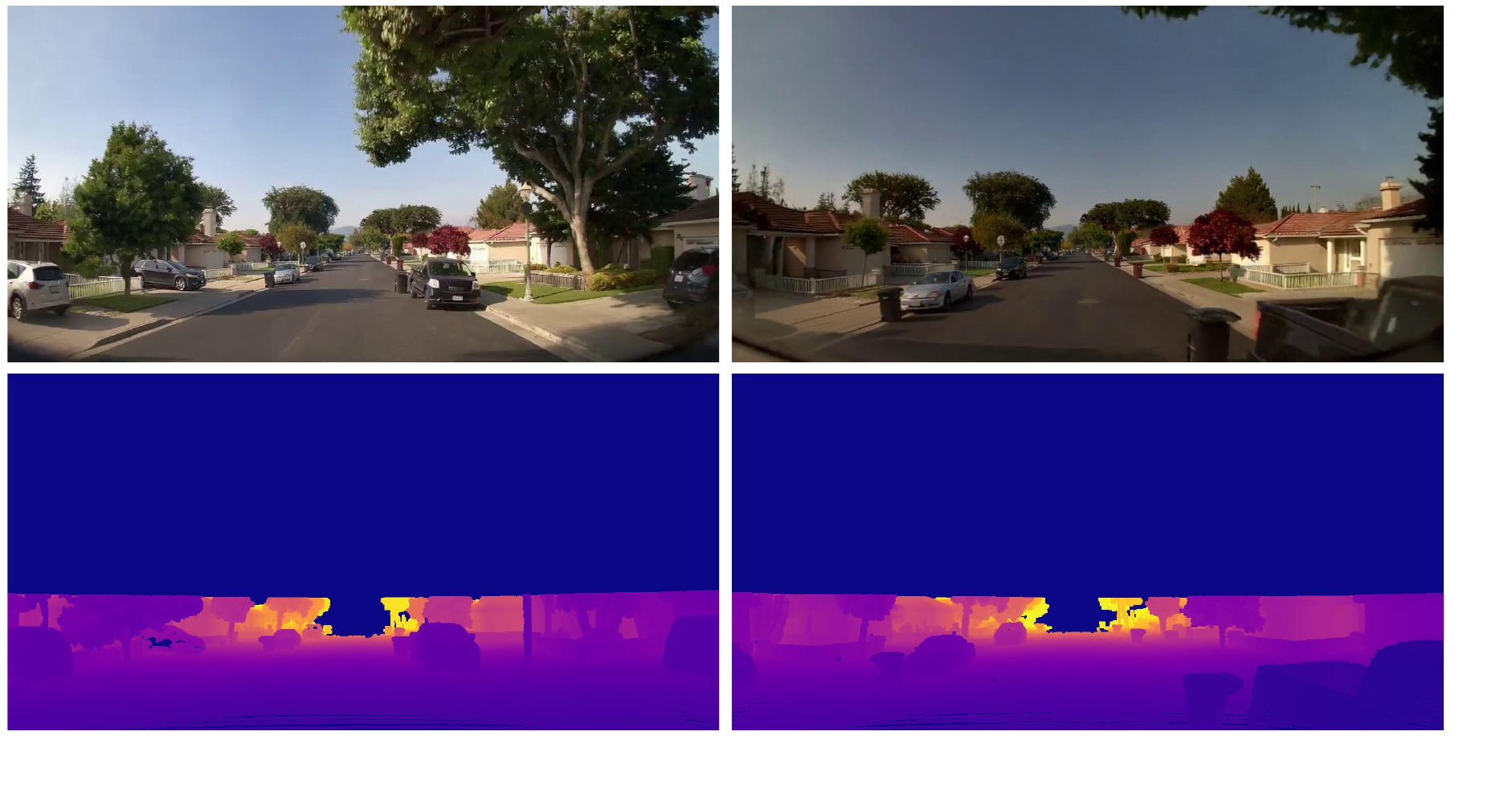}
            \caption{Example 2}
        \end{subfigure}

        \vspace{2pt}

        \begin{subfigure}{0.445\textwidth}
            \centering
            \includegraphics[
                page=10,
                width=\linewidth,
                trim={0mm 20mm 20mm 0mm},
                clip
            ]{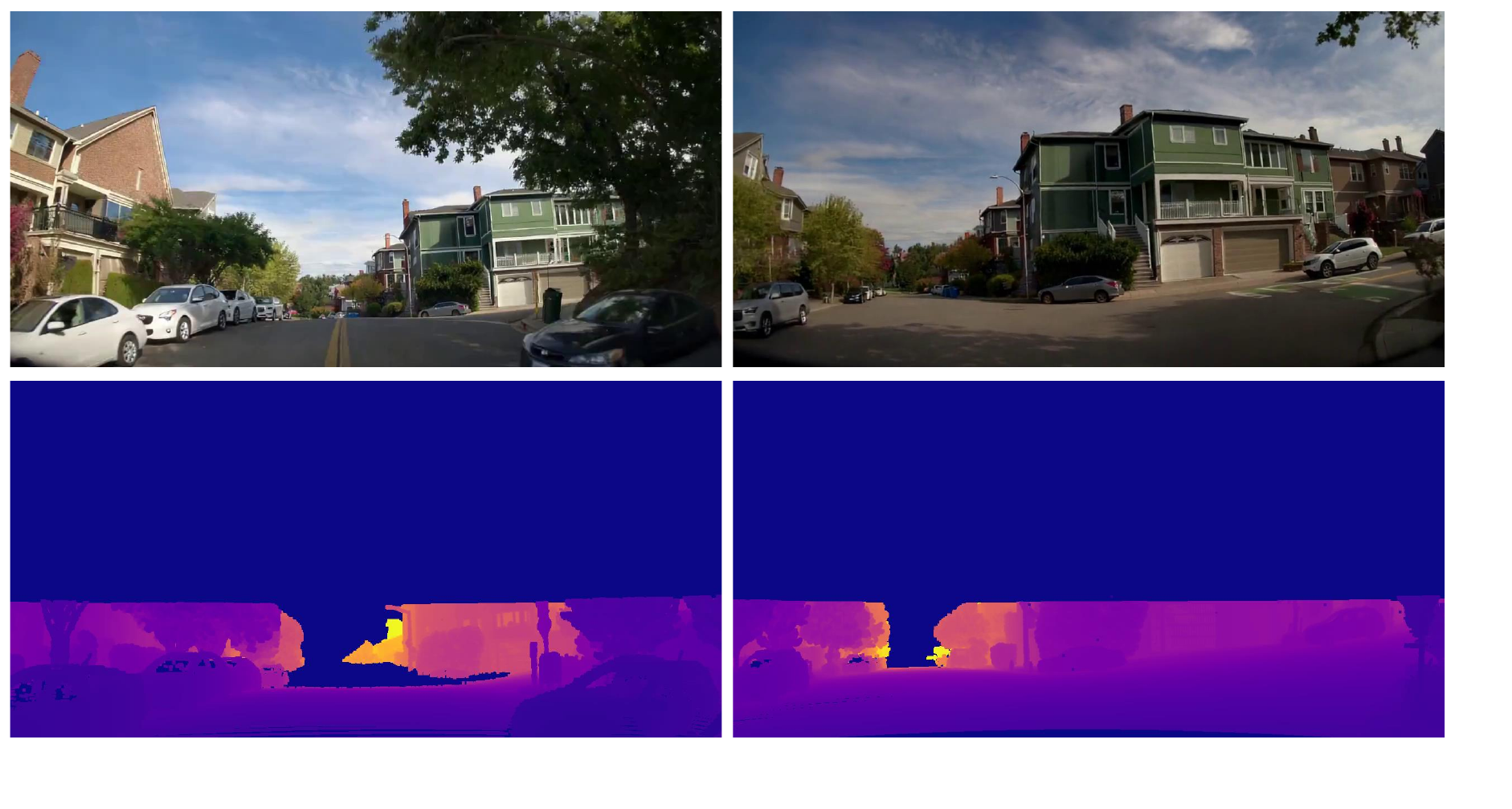}
            \caption{Example 3}
        \end{subfigure}
        \hfill
        \begin{subfigure}{0.445\textwidth}
            \centering
            \includegraphics[
                page=11,
                width=\linewidth,
                trim={0mm 20mm 20mm 0mm},
                clip
            ]{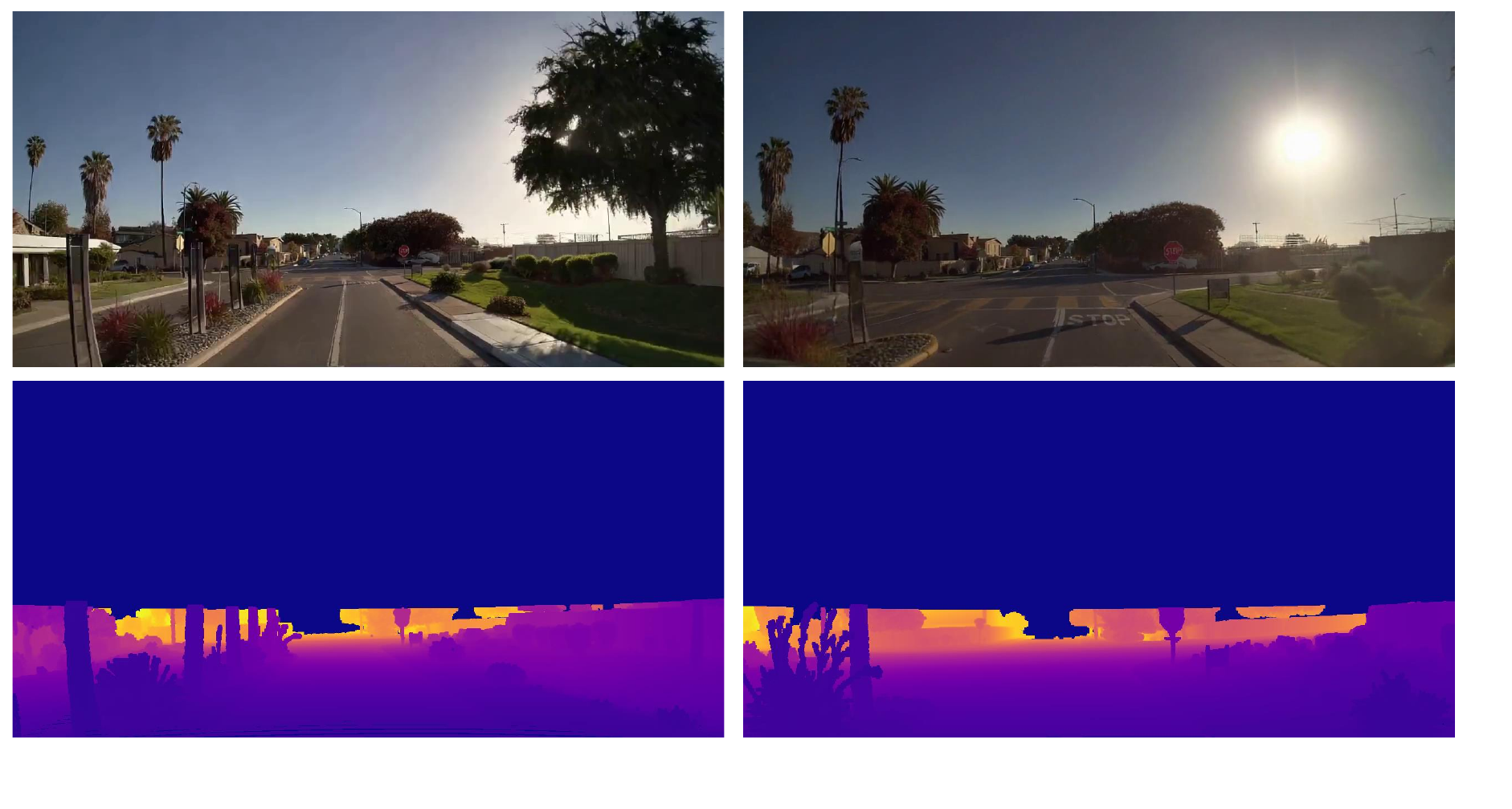}
            \caption{Example 4}
        \end{subfigure}

        \captionof{figure}{
            \textbf{Multi-view consistent RGB generation from depth on Waymo.}
            Four representative synthetic RGB examples generated from LiDAR-projected depth.
            The synthesized frames preserve scene geometry across viewpoints.
        }
        \label{fig:gen_rgb_waymo}
    \end{minipage}
\end{figure*}

\end{document}